\documentclass{article}

\usepackage{microtype}
\usepackage{graphicx}
\usepackage{subcaption}
\usepackage{booktabs} 

\usepackage{hyperref}

\usepackage[accepted]{icml2026}

\usepackage{amsmath}
\usepackage{amssymb}
\usepackage{mathtools}
\usepackage{amsthm}

\usepackage{amsmath,amsfonts,bm}

\def\eqref#1{equation~\ref{#1}}

\def\1{\bm{1}}

\DeclareMathAlphabet{\mathsfit}{\encodingdefault}{\sfdefault}{m}{sl}
\SetMathAlphabet{\mathsfit}{bold}{\encodingdefault}{\sfdefault}{bx}{n}

\newcommand{\E}{\mathbb{E}}

\DeclareMathOperator{\End}{End}

\newcommand{\C}{\mathbb{C}}

\usepackage{hyperref}
\usepackage{url}
\usepackage{booktabs,multirow,makecell, graphicx}
\usepackage{mathtools}
\usepackage{tikz-cd}
\usepackage{subcaption}
\usepackage{adjustbox}
\usepackage[most]{tcolorbox}
\usepackage{minted}
\usepackage{xcolor}
\usepackage{pdfpages}
\usepackage{algpseudocode}
\usepackage{float} 

\usepackage{listings}
\usepackage[T1]{fontenc}
\usepackage{inconsolata}

\makeatletter
\@ifundefined{submissiontrue}{}{%
  
  \renewcommand{\textcolor}[2]{#2} 
}
\makeatother

\tcbuselibrary{listings,skins,breakable}

\definecolor{mygreen}{rgb}{0,0.5,0}
\definecolor{myblue}{rgb}{0,0,1}
\definecolor{myorange}{rgb}{1,0.5,0}
\definecolor{myred}{rgb}{0.8,0,0}

\definecolor{codeframe}{RGB}{236,146,72}   
\definecolor{codebg}{RGB}{255,255,255}     
\definecolor{codegutter}{RGB}{253,236,220} 
\definecolor{codenums}{RGB}{110,110,110}   
\definecolor{codekw}{RGB}{54,152,65}       
\definecolor{codebool}{RGB}{155,78,170}    
\definecolor{codecomment}{RGB}{74,156,158} 
\definecolor{codestr}{RGB}{180,90,45}      

\definecolor{codegreen}{rgb}{0,0.5,0}
\definecolor{codered}{rgb}{0.8,0.2,0.2}
\definecolor{codeblue}{rgb}{0.2,0.2,0.7}

\lstdefinestyle{mypython}{
  language=Python,
  basicstyle=\ttfamily\small,
  keywordstyle=\bfseries\color{codegreen},
  commentstyle=\itshape\color{codered},
  stringstyle=\color{codeblue},
  numbers=left,
  numberstyle=\scriptsize,
  numbersep=8pt,
  showstringspaces=false,
  keepspaces=true,
  columns=fullflexible,
  tabsize=4,
  breaklines=true,
  frame=none,
  xleftmargin=1em
}

\lstdefinestyle{py-screenshot}{
  language=Python,
  basicstyle=\ttfamily\footnotesize,
  numbers=left,
  numberstyle=\ttfamily\scriptsize\color{codenums},
  numbersep=8pt,
  xleftmargin=2.2em,
  breaklines=true,
  postbreak=\mbox{\textcolor{codenums}{$\hookrightarrow$}\space},
  showstringspaces=false,
  tabsize=4,
  keepspaces=true,
  columns=fullflexible,
  commentstyle=\itshape\color{codecomment},
  stringstyle=\color{codestr},
  keywordstyle=\color{codekw}\bfseries,
  keywordstyle=[2]\color{codebool}\bfseries,
  morekeywords=[2]{and,or,not},
}

\tcbset{
  codebox/.style={
    enhanced,
    breakable,
    listing only,
    listing engine=listings,
    listing options={style=py-screenshot},
    colback=codebg,
    colframe=codeframe,
    arc=3mm,
    boxrule=0.9pt,
    left=1mm,right=1mm,top=1mm,bottom=1mm,
    fonttitle=\bfseries\ttfamily\footnotesize,
    coltitle=black,
    attach boxed title to top left={xshift=2mm,yshift=-2mm},
    boxed title style={
      colback=codebg,
      colframe=codeframe,
      arc=2mm,
      boxrule=0.9pt,
      left=2mm,right=2mm,top=0.5mm,bottom=0.5mm,
    },
    underlay middle={
      \begin{tcbclipinterior}
        \fill[codegutter]
          (interior.south west) rectangle ([xshift=3.3em]interior.north west);
      \end{tcbclipinterior}
    
      \begin{tcbclipframe}
        \path[clip,rounded corners=2mm]
          (title.south west) rectangle (title.north east);
        \fill[codegutter]
          (frame.west|-title.south) rectangle ([xshift=3.3em]frame.west|-title.north);
      \end{tcbclipframe}
    },
  }
}

\usetikzlibrary{backgrounds} 

\usepackage[capitalize,noabbrev]{cleveref}

\theoremstyle{plain}
\newtheorem{theorem}{Theorem}[section]

\newtheorem{lemma}[theorem]{Lemma}
\newtheorem{corollary}[theorem]{Corollary}
\theoremstyle{definition}
\newtheorem{definition}[theorem]{Definition}

\theoremstyle{remark}

\usepackage[textsize=tiny]{todonotes}

\icmltitlerunning{Approximate Equivariance via Projection-Based Regularisation}

\begin{document}

\twocolumn[
  \icmltitle{Approximate Equivariance via Projection-Based Regularisation}



  \icmlsetsymbol{equal}{*}

  \begin{icmlauthorlist}
    \icmlauthor{Torben Berndt}{hits}
    \icmlauthor{Jan Stühmer}{hits,kit}
  \end{icmlauthorlist}

  \icmlaffiliation{hits}{Heidelberg Institute for Theoretical Studies, Heidelberg, Germany}
  \icmlaffiliation{kit}{IAR, Karlsruhe Institute of Technology, Karlsruhe, Germany}


  \icmlkeywords{Machine Learning, ICML}

  \vskip 0.3in
]



\printAffiliationsAndNotice{}  

\begin{abstract}

   Equivariance is a powerful inductive bias in neural networks, improving generalisation and physical consistency. Recently, however, non-equivariant models have regained attention, due to their better runtime performance and imperfect symmetries that might arise in real-world applications. This has motivated the development of approximately equivariant models that strike a middle ground between respecting symmetries and fitting the data distribution. Existing approaches in this field either rely on sampling from a group, incurring a high sample complexity, or explicitly parameterise a model as a sum of an equivariant and non-equivariant network.
   This work instead approaches approximate equivariance via a projection-based regulariser which leverages a layer-wise orthogonal decomposition of a network's layers into equivariant and non-equivariant components. In contrast to existing methods, this penalises non-equivariance at an operator level across the full group orbit, rather than point-wise as in sample-based approaches. We present a mathematical framework for computing the non-equivariance penalty exactly and efficiently in both the spatial and spectral domains. \textcolor{blue}{In our experiments, our method is competitive with prior approximate-equivariance approaches in task performance, while achieving substantial runtime gains over sample-based regularisers.}
   \footnote{The source code is available at \url{https://github.com/hits-mli/approximate-equivariance-projection}.}
    
    




    
\end{abstract}   

\section{Introduction}

Over the past few years, equivariance has been proven to be a powerful design principle for machine learning models across chemistry~\citep{thomas2018tensor,satorras2021gnn,brandstetter2022geometric,hoogeboom22a,xu2024equivariant},
physics~\citep{bogatskiy2020,spinner2024,brehmer2025lorentzequivarianttransformerlhc}, robotics~\citep{hoang2025geometryaware}, and engineering~\citep{toshev2023}. 
Recently, however, there has been a shift back towards non-equivariant models, most prominently AlphaFold-3~\citep{abramson2024accurate}. Non-equivariant architectures often allow more flexible feature parameterisations and can be easier to optimise as the search is not restricted to an equivariant hypothesis class. This broader parameter space may enable the optimiser to find better minima than if it was confined to strictly equivariant models \citep{Pertigkiozoglou2025}. Moreover, many existing equivariant architectures rely on specialised tensor products to preserve symmetry~\citep{weiler2019general,brandstetter2022geometric}, which can be less efficient to compute on modern GPUs than dense matrix–vector operations.

At the same time, recent work demonstrates that equivariance remains a valuable inductive bias even at scale~\citep{brehmer2025does}, and, for example, state-of-the-art molecular property prediction models continue to leverage it~\citep{liao2023equiformer,equiformer_v2,fu2025learning}. This motivates approaches that retain the benefits of equivariance without incurring its full constraints or computational costs.
A common approach is to promote equivariance in otherwise non-equivariant architectures at the level of samples - for example via data augmentation, as in AlphaFold-3~\citep{abramson2024accurate}, or pointwise equivariance penalties~\citep{Bai_2025_CVPR}. In this work, we take a different perspective and introduce \emph{projection-based equivariance regularisation}, a framework based on first-principles which allows tuning equivariance into any neural architecture on the operator level, thereby directly affecting the model weights. Our primary contributions are:
%
%
    (i)~We propose a theoretically-grounded approach to regularise general machine learning models towards exact equivariance.
    (ii)~Making use of the orthogonal decompostion of functions into equivariant and non-equivariant components, we are able to penalise non-equivariance on an operator level over the whole group orbit.
    (iii)~We show how to efficiently calculate the closed-form projection by working in the Fourier domain, allowing efficient regularisation for continuous groups such as $SO(n)$.
    (iv)~\textcolor{blue}{We empirically show that our method is competitive with existing approximate-equivariance approaches in task performance, while offering especially large runtime gains over sample-based regularisers.}

\subsection{Related Work}

A growing body of work relaxes strictly equivariant architectures to better capture approximate or imperfect symmetries in data. \citet{finzi2021residual} model departures from symmetry by adding a small non-equivariant ``residual'' pathway to an otherwise equivariant network. This however, leads to a double parameterisation of the equivariant part, as the residual pathway also models equivariant functions. 
\textcolor{blue}{
\citet{Kim2023} extend this to mixed symmetries.
}
\citet{romero2022learning} introduce partial group convolutions that activate only on a subset of group elements. For discrete groups, \citet{Wang2022} propose relaxed group convolutions, later extended by \citet{Wang2024} to expose symmetry-breaking mechanisms; \citet{hofgard2024relaxedequivariantgraphneural} further generalise this framework to continuous groups. 
\citet{veefkind24} introduce a learnable non-uniform measure over the group within steerable CNNs, yielding partially equivariant SCNNs whose degree of symmetry breaking is explicitly encoded in the learned measure. This method however only addresses the special case of steerable CNNs, while we propose an architecture-agnostic regulariser. \citet{samudre25a} enforce approximate equivariance through group-matrix–structured convolutional layers with low displacement rank, so that symmetry and its controlled violation are encoded as proximity to the group-matrix manifold, leading to highly parameter-efficient CNNs for discrete groups.
\citet{mcneela2024equivarianceliealgebraconvolutions} introduce Lie-algebra convolutions with a non-strict equivariance bias, and \citet{van2022relaxing} relax translation equivariance using spatially non-stationary convolution kernels. On graphs, \citet{Huang2023} develop approximately automorphism-equivariant GNNs. A complementary line of work studies how to measure equivariance (or its violation) and use it in training objectives.
\textcolor{blue}{A number of works penalise pointwise deviations from equivariance constraints using randomly sampled group transformations: \citet{Bai_2025_CVPR} use this strategy for artifact reduction in imaging, \citet{kouzelis2025eqvae} for VAEs in generative modelling, and \citet{zhong2023improving} for depth and normal prediction in images.}
Finally, assuming a model splits into equivariant and non-equivariant parts, \citet{Pertigkiozoglou2025} propose to phase out the non-equivariant component during training. \citet{manolache2025learning} extend this idea by controlling the trade-off via constrained optimisation.

\section{Background}

\paragraph{Notation.}
For inner-product vector spaces $V$ and $V^\prime$, we denote the \textit{identity} on $V$ by $I_V$ and write $\mathrm{Hom}(V, V^\prime)$ for the vector space of linear maps $V\to V^\prime$, with $\mathrm{End}(V):=\mathrm{Hom}(V,V)$. We write $T^*:V^\prime\to V$ for the adjoint and $U(V)=\{T\in\mathrm{End}(V):TT^*=T^*T=I_V\}$ for the unitary group. For a measure $\mu$ on $V$ and an inner product $\langle\cdot,\cdot\rangle_{V^\prime}$ on $V^\prime$, we use the $L^2(\mu)$ inner product on functions $S,T:V\to V^\prime$ given by $\langle S,T\rangle_\mu:=\int_V\langle S(v),T(v)\rangle_{V^\prime}\,d\mu(v)$ with norm $\|T\|_\mu^2=\langle T,T\rangle_\mu$. For $T\in\mathrm{Hom}(V,V')$, if $\mu$ is an isotropic Gaussian on $V$, then $\|T\|_{\mu}^2=\int_V\|T(v)\|_{V'}^2\,d\mu(v)=\mathrm{tr}(T^*T)\eqqcolon\|T\|_{\mathrm{HS}}^2$, which in finite dimensions is equal to the Frobenius norm $\|T\|_F=\sqrt{\sum_{i,j}|T_{ij}|^2}$. 

\paragraph{Unitary representations.}
Given a group $G$, a \emph{unitary representation} is a homomorphism $\pi:G\!\to\!U(V_\pi)$ into the unitary operators on a Hilbert space $V_\pi$; we call the pair $(V_\pi,\pi)$ a $G$-module. Two representations $\pi:G\!\to\!U(V_\pi)$ and $\pi':G\!\to\!U(V_{\pi'})$ are said to be \emph{isomorphic} if there exists a unitary $U:V_\pi\!\to\!V_{\pi'}$ with
\(
\pi(g)=U\,\pi'(g)\,U^{-1}
\)
for all $g\in G$. A representation is \textit{irreducible} if it is not isomorphic to a direct sum of non-zero representations $\pi \oplus \pi^\prime$ where $\pi \oplus \pi^\prime:G \rightarrow U(V \oplus V^\prime)$ is defined by $(\pi \oplus \pi^\prime)(g)(v, v^\prime) = (\pi(g)v, \pi^\prime(g)v^\prime)$.

\textbf{Haar measure.} Let $G$ be a compact group. The \textit{Haar measure} $\lambda$ is the unique \textit{bi-invariant} and \textit{normalised} measure, i.e. for all Borel sets $E \subset G$ and every $g \in G$ we have $\lambda(gE) = \lambda(Eg) = \lambda(E)$, and $\lambda(G) = 1$. We can view the Haar measure as a uniform distribution over the group $G$. Indeed, if $G$ is discrete, the Haar measure becomes the discrete uniform measure with $\lambda(\{g\}) = \frac{1}{|G|}$ for all $g \in G$. 


\paragraph{Equivariance and $G$-smoothing.}
Let $T:(V,\pi)\!\to\!(V',\pi')$ be a (bounded) linear map between $G$-modules. We say $T$ is \emph{$G$-equivariant} if
\(
T\big(\pi(g)v\big)=\pi'(g)\,T(v)
\)
for all $g\in G$, $v\in V$. If the action on $V'$ is trivial ($\pi'(g)=I_{V'}$), we call $T$ \emph{invariant}.
Averaging over $G$ yields the \emph{$G$-smoothing (Reynolds) operator}
\begin{align}\label{eqn:projection_operator}
  P(T)\;=\;\int_G \pi'(g)^{*}\, T \,\pi(g)\, d\lambda(g).
\end{align}

\paragraph{Projection onto the equivariant subspace.}
When $\pi,\pi'$ are unitary, $P$ is the orthogonal projector (with respect to the $L^2(\mu)$ inner product) from $\mathrm{Hom}(V,V')$ onto the closed subspace of $G$-equivariant linear maps \citep{elesedy2021provably}. The following structural decomposition will be useful.

\begin{lemma}[\citet{elesedy2021provably}, Lemma 1]\label{lem:reynolds-decomp}
Let $\mathcal{H}\subset \{(V,\pi)\!\to\!(V',\pi')\}$ be a function space that is closed under $P$ (i.e.\ $P(T)\in\mathcal{H}$ whenever $T\in\mathcal{H}$). Define
\begin{align}
    S \;=\; \{T\in\mathcal{H}:\ T\ \text{is $G$-equivariant}\}, \\
    A \;=\; \ker P \;=\; \{T\in\mathcal{H}:\ P(T)=0\}.
\end{align}
Then $P$ is an orthogonal projection with range $S$ and kernel $A$, and hence
\(
\mathcal{H}=S\oplus A.
\)
\end{lemma}

In particular, every $T\in\mathcal{H}$ orthogonally decomposes uniquely as $T=P(T)+\big(T-P(T)\big)$, where $P(T)$ is the $G$-equivariant component $S$ and $T-P(T)\in A$ is its $G$-anti-symmetric component. Moreover, we have the following:
\begin{corollary}
    A function $T:(V,\pi)\!\to\!(V',\pi')$ is $G$-equivariant if and only if $P(T)=T$.
\end{corollary}

\begin{figure}[t]
  \centering
  \begin{minipage}[t]{\linewidth}
  \includegraphics[width=\linewidth]{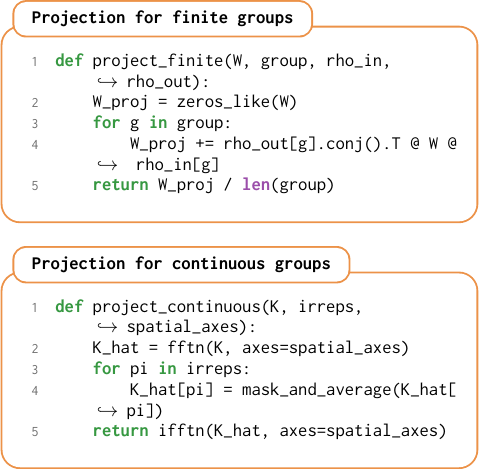}
  \end{minipage}
  \caption{Pseudo-code for the equivariant projection for finite (left) and continuous groups (right).}
  \label{alg:projection}
\end{figure}

\section{Equivariant Projection Regularisation}\label{sec:equiv_projection_regularisation}

Motivated by these observations, we propose a simple framework for learning (approximately) equivariant models: Let $\mathcal{H}$ be a hypothesis class and $L_\textrm{task}(T)$ a task-specific loss function for $T \in \mathcal{H}$. We learn $T$ by solving
\begin{align}
    T^* \in \arg\inf_{T \in \mathcal{H}} \; L_\textrm{task}(T) 
    &\;+\; \lambda_G \,\|P(T)\|_\mu^2 \\
    &\;+\; \lambda_\perp \,\|T - P(T)\|_\mu^2,
\end{align}
where $\lambda_G,\lambda_\perp \ge 0$ are hyperparameters. Intuitively, increasing $\lambda_\perp$
penalises $\|T - P(T)\|_\mu^2$ more strongly, which encourages $P(T) = T$, steering the solution toward stronger equivariance according to Lemma~\ref{lem:reynolds-decomp}. The hyperparameter $\lambda_G$ controls regularisation of the equivariant part. Since $P(T)$ and $T-P(T)$ are orthogonal components, we have $\|T\|_\mu^2=\|P(T)\|_\mu^2+\|T-P(T)\|_\mu^2$. Hence, if $T$ is linear and for $\lambda_G=\lambda_\perp$ this reduces to standard Frobenius weight regularisation $\lambda_G\|T\|_F^2$.

In what follows, we provide a theoretical justification for using $\|T - P(T)\|_\mu$ as a regulariser towards stronger equivariance. Recalling that $P(T)$ denotes the closest equivariant operator to $T$, we show that the distance $\|T - P(T)\|_\mu$ is quantitatively equivalent to a natural measure of non-equivariance, the \emph{equivariance defect}.

\subsection{Bounding the Equivariance Error}

\begin{definition}[Equivariance defect]
    Let $T$ be a function between $G$-modules with actions $\pi_{\mathrm{in}}$ and $\pi_{\mathrm{out}}$. The \textit{equivariance defect} at $g\in G$ is
    \begin{equation}
      \Delta_g(T) \;\coloneqq\; \pi_{\mathrm{out}}(g)\circ T \;-\; T\circ \pi_{\mathrm{in}}(g),
    \end{equation}
    and the \textit{worst-case defect} is
    \begin{equation}
      \mathcal{E}(T) \;\coloneqq\; \sup_{g\in G}\, \|\Delta_g(T)\|_\mu.
    \end{equation}
\end{definition}

\noindent
By Lemma~\ref{lem:reynolds-decomp} \citep{elesedy2021provably}, the quantity $\mathcal{E}(T)$ vanishes if and only if $T$ is $G$-equivariant. The next lemma shows that this defect is effectively controlled, up to constants, by the distance to the equivariant subspace measured by the projection $P$.

\begin{lemma}\label{lem:proj-defect}
    For every (Lipschitz) function $T$ between $G$-modules with unitary actions,
    \begin{equation}
        \|T- P(T)\|_\mu \;\le\; \mathcal{E}(T) \;\le\; 2\,\|T- P(T)\|_\mu.
    \end{equation}
\end{lemma}
\begin{proof}
    See Appendix \ref{app:proof_lemma_proj_defect}
\end{proof}

Lemma~\ref{lem:proj-defect} shows that regularising by $\mathcal{E}(T)$ or by $\|T-P(T)\|$ is equivalent up to a factor of~$2$. Thus, minimising $\|T - P(T)\|$ minimises the worst-case defect.

\noindent
In practice, $T$ will be some type of neural network architecture and is hence a composition of functions. The following bound decomposes the global defect of a network into per-layer defects, weighted by downstream Lipschitz constants. 

\begin{lemma}\label{lem:defect_composition}
    Let $T = f_k\circ f_{k-1}\circ\cdots\circ f_1$ be a composition of Lipschitz maps between $G$-modules with unitary actions, and set $L_m := \mathrm{Lip}(f_m)$.
    Then
    \begin{equation}\label{eqn:decomposition}
        \mathcal{E}(T) \;\le\; \sum_{i=1}^k \Big(\prod_{m\neq i}^{k} L_m\Big)\, \mathcal{E}(f_i).
    \end{equation}
\end{lemma}
\begin{proof}
    See Appendix \ref{app:proof_lemma_defect_composition}
\end{proof}

The bound above immediately yields the following corollary for standard feed-forward networks first shown by \citet{Kim2023}.

\begin{corollary}[\citet{Kim2023}]\label{cor:bound_neural_net}
    Let
    \begin{equation}
      T \;=\; W^{(S)} \circ \sigma_{S-1} \circ W^{(S-1)} \circ \cdots \circ \sigma_{1} \circ W^{(1)}
    \end{equation}
    be an $S$-layer network where each linear map $W^{(l)}$ acts between $G$-modules with unitary actions and each activation $\sigma_{l}$ is $G$-equivariant and Lipschitz. Then
    \begin{equation}
      \mathcal{E}(T)
      \;\le\;
      C \,\sum_{l=1}^{S} \bigl\|\,W^{(l)} - P\big(W^{(l)}\big)\,\bigr\|_F,
    \end{equation}
    for a constant $C>0$ depending only on the operator norms of the $W^{(l)}$, the Lipschitz constants of the $\sigma_l$, and (when working on a bounded input domain) its radius.
\end{corollary}
\begin{proof}
    See Appendix \ref{app:proof_cor_bound_neural_net}.
\end{proof}

\subsection{Projection in Fourier Space}

The previous section motivates the use of the norm of the projection operator as a regulariser. When the projection operator in Equation~\ref{eqn:projection_operator} is efficiently computable in the spatial domain, e.g., for small finite groups (see Section~\ref{sec:experiments_discrete_rotation_equivariance}), this is straightforward; the algorithm in Figure~\ref{alg:projection} provides pseudo-code for this case. However, in many applications, the group is large (for instance, uncountably infinite, as in $SO(n)$, the group of rotations about the origin in $\mathbb{R}^n$; see Sections~\ref{sec:so2_equivariance} and 
\ref{sec:medmnist}). 
In such cases, the integral in Equation~\ref{eqn:projection_operator} rarely admits a closed-form solution.

We therefore switch to the spectral domain. We assume the following setup, which is in line with the geometric deep learning blueprint~\citep{bronstein2021geometric} that constructs equivariant networks as a composition of equivariant linear layers with equivariance-preserving non-linearities. Let $G$ be a compact group with normalised Haar measure $\lambda$, and consider linear maps $T:L^2(G)\!\to\!L^2(G)$ on the Hilbert space of square-integrable complex functions $L^2(G)=\{f:G\rightarrow\mathbb{C}\}$ with inner product
\begin{align}
  \langle f,h\rangle=\int_G f(g)\,\overline{h(g)}\,d\lambda(g).
\end{align}
We study equivariance with respect to the (left) regular representation $\tau:G\to U(L^2(G))$ defined by
\begin{equation}
  (\tau(g)f)(x)=f(g^{-1}x),\qquad x,g\in G.
\end{equation}

\begin{figure}[t]
  \centering
  \begin{adjustbox}{max width=\linewidth}
  \begin{tikzcd}[
      row sep=3.8em,
      column sep=6em,
      cells={nodes={font=\small}}
    ]
      T : L^2(G,V_{\mathrm{in}}) \to L^2(G,V_{\mathrm{out}})
        \arrow[r, "P_{\mathrm{equiv}}" {name=Top}]
        \arrow[d, "Peter\text{--}Weyl"']
      & T_{\mathrm{equiv}} : L^2(G,V_{\mathrm{in}}) \to L^2(G,V_{\mathrm{out}})
        \arrow[d, "Peter\text{--}Weyl"]
      \\
      \widehat T \;=\; \big\{\widehat T(\pi,\sigma)\big\}_{\pi,\sigma\in\widehat G}
        \arrow[u]
        \arrow[r, "P:\; \{T_{\pi\sigma}\}\mapsto\{\delta_{\pi,\sigma}\,\mathrm{Av}_\pi(T_{\pi\pi})\}"'{yshift=-1.5ex}]
      & \widehat{T_{\mathrm{equiv}}}
        \;=\; \bigoplus_{\pi\in\widehat G} \!\big(I_{V_\pi}\otimes W_\pi\big)
        \arrow[u]
      \arrow[ul, phantom, "\scalebox{1.2}{$\circlearrowleft$}" description]
    \end{tikzcd}
  \end{adjustbox}
  \caption{Commutative diagrams showing how to apply the projection operator in Fourier space. We zero out off-diagonals \((\pi\neq\sigma)\) and average within each frequency block to obtain \(I_{V_\pi}\otimes B_\pi\).}
  \label{fig:commutative_diagrams}
\end{figure}

\begin{figure*}
    \centering
    \includegraphics[width=\linewidth]{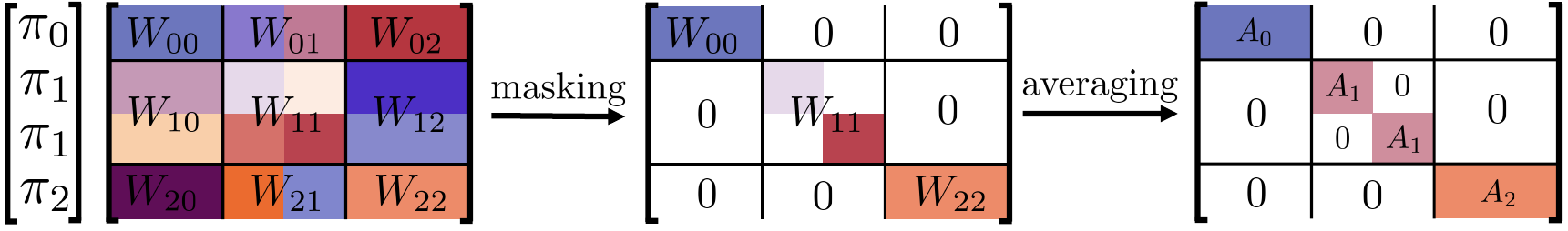}
    \caption{
    Projection of a linear map onto the equivariant subspace.
    The initial dense weight matrix mixes different irreducible grades
    \(\pi_0,\pi_1,\pi_2\), with \(\operatorname{dim}\pi_1 = 2\).
    The masking step removes all couplings between inequivalent grades,
    retaining only the isotypic blocks.
    Averaging then projects each retained block onto the equivariant form
    \(I_{\pi_i}\otimes A_i\), where \(A_i\) acts on the corresponding
    multiplicity space.
    }
    \label{fig:irreps_projection_sketch}
\end{figure*}

We denote by $\widehat{G}$ the set of equivalence classes of finite-dimensional irreducible representations of $G$ and call it the \emph{unitary dual} of $G$. Each $[\pi]\in\widehat G$ has a representative $\pi:G\to U(V_\pi)$ with $d_\pi=\dim V_\pi$. For $f\in L^2(G)$,  we define the \textit{(non-abelian) Fourier transform} as
\begin{equation}
    \widehat f_\pi\;:=\;\int_G f(g)\,\pi(g)^{*}\,d \lambda (g)\;\in\text{End}(V_\pi).
\end{equation}

\textcolor{blue}{
By the Peter--Weyl theorem~\citep{PeterWeyl1927}, these coefficients form a complete spectral representation of $f$. More explicitly, the Fourier inversion formula reconstructs scalar functions as
\begin{equation}
    f(g)
    =
    \sum_{\pi\in\widehat G}
    d_\pi\,
    \mathrm{tr}\!\left(
        \widehat f(\pi)\,\pi(g)
    \right),
\end{equation}
with convergence in $L^2(G)$, and pointwise under the usual additional regularity assumptions. Thus, passing to Fourier space retains all information about functions.
}

\textcolor{blue}{
Under the Peter--Weyl identification
\begin{equation}
    L^2(G)
    \cong
    \bigoplus_{\pi\in\widehat G} V_\pi\otimes V_\pi^*,
\end{equation}
a general linear operator $T:L^2(G)\to L^2(G)$ decomposes into frequency blocks
\begin{equation}
    \widehat T_{\pi\sigma}
    :
    V_\sigma\otimes V_\sigma^*
    \to
    V_\pi\otimes V_\pi^*.
\end{equation}
Equivalently, its action on Fourier coefficients can be written as
\begin{equation}
    \widehat{(Tf)}(\pi)
    =
    \sum_{\sigma\in\widehat G}
    \widehat T_{\pi\sigma}\,\widehat f(\sigma).
\end{equation}
A general operator may therefore mix different frequencies $\sigma\to\pi$. By Schur's lemma \citep{schur1905neue}, equivariance imposes a block-diagonal structure on these blocks: it forbids mixing between inequivalent irreducible components and restricts the surviving diagonal blocks. 
}
\begin{theorem}[Informal]
  \label{thm:equiv_decomp_informal}
  Equivariant linear maps are block-diagonal in the frequency domain (one block per irreducible representation). Hence, the projection onto equivariant subspaces acts by zeroing out all off-diagonal terms and averaging within.
\end{theorem}
\textcolor{blue}{
We schematically depict how the projection acts on weight matrices in Figure \ref{fig:irreps_projection_sketch}.
}
Hence, whenever an efficient Fourier transform is available (e.g., on regular grids) or the model is already parameterised spectrally (e.g., eSEN \citep{fu2025learning}), imposing equivariance reduces to diagonalising the relevant linear operators in the spectral domain. The next section makes this more mathematically precise.

\subsection{Equivariant Maps are Diagonal Across Frequencies}

\begin{theorem}\label{thm:equiv_fourier}
    Let $T: L^2(G) \rightarrow  L^2(G)$ be a linear function which is equivariant with respect to the (left) regular representation, i.e. $\tau(g) \circ T = T \circ \tau(g)$ for all $g \in G$. Then $T$ decomposes as follows:
    \begin{equation}\label{eqn:T_decomposition}
        \widehat T \; \cong \; \bigoplus_{\pi\in\widehat G} I_{V_\pi} \otimes B_\pi
    \end{equation}
    for some $B_\pi \in \text{End}(V^*_\pi)$ (one for each $\pi$).
\end{theorem}
\begin{proof}
    See Appendix \ref{app:proof_thm_equiv}.
\end{proof}

This means that an equivariant linear map $T$ does not mix between irreps; it is block-diagonal. We now show what this means for the projection of a general linear operator $T$.

\begin{corollary}\label{cor:right-mult}
  Let $T:L^2(G) \rightarrow  L^2(G)$ be linear and set ${P_\textrm{equiv}(T)}$ to be its equivariant projection. Then for each $[\pi] \in \widehat G$, there exists $B_\pi \in \text{End}(V^*_\pi)$ such that for all $f \in L^2(G)$,
  \begin{equation}
        \widehat {P(T)(f)}(\pi) = \widehat f(\pi)\,B_\pi .
  \end{equation}
\end{corollary}


\subsection{Vector-valued Signals and Fiber-wise Projection}

Thus far we treated scalar signals $f\!\in\!L^2(G)$. In many applications (e.g.\ steerable CNNs~\cite{cohen2017steerable}, tensor fields) one works with vector-valued signals taking values in a finite-dimensional unitary $G$-module $(V,\rho)$.
Define $L^2(G,V) \;\cong\; L^2(G)\otimes V$ with action
\begin{align}
  \big((\tau\!\otimes\!\rho)(g)f\big)(x)=\rho(g)\,f(g^{-1}x).
\end{align}
More generally, for an operator $T:L^2(G,V_{\mathrm{in}})\!\to\!L^2(G,V_{\mathrm{out}})$ we measure equivariance with respect to the pair of actions
\(
  \tau\!\otimes\!\rho_{\mathrm{in}}
\)
(on the domain)
and
\(
  \tau\!\otimes\!\rho_{\mathrm{out}}
\)
(on the codomain), i.e.
\begin{equation}
  (\tau\!\otimes\!\rho_{\mathrm{out}})(g)\,\circ\,T \;=\; T\,\circ\,(\tau\!\otimes\!\rho_{\mathrm{in}})(g)\qquad\forall g\in G.
\end{equation}

\textcolor{blue}{
Under the identification
\[
L^2(G,V)\cong \bigoplus_{\pi\in\widehat G} V_\pi\otimes V_\pi^*\otimes V,
\]
a general operator
\(T:L^2(G,V_{\mathrm{in}})\to L^2(G,V_{\mathrm{out}})\)
has blocks
\[
\widehat T_{\pi\sigma}
:
V_\sigma\otimes V_\sigma^*\otimes V_{\mathrm{in}}
\to
V_\pi\otimes V_\pi^*\otimes V_{\mathrm{out}}.
\]
As before, the equivariant projection again removes all \(\pi\neq\sigma\) blocks and
projects each diagonal block onto the appropriate intertwiner space, leading to the following theorem. Details are provided in Appendix \ref{app:vector-valued_signals}.
}


\begin{theorem}\label{thm:vector_block}
    Let $T:L^2(G,V_{\mathrm{in}})\!\to\!L^2(G,V_{\mathrm{out}})$ be linear. Then the equivariant projection decomposes as 
    \begin{align}
      \widehat{P_{\mathrm{equiv}}(T)}
      &\;\cong\;
      \bigoplus_{\pi\in\widehat G}
      \Big(I_{V_\pi}\otimes W_\pi\Big)
    \end{align}
    with\\
\resizebox{.9\linewidth}{!}{
  \begin{minipage}{\linewidth}
        \begin{align}
      W_\pi 
      &\;=\;
      \int_G
      \big(\pi(g)^{*}\!\otimes\!\rho_{\mathrm{out}}(g)\big)\,
      \widehat T(\pi,\pi)\,
      \big(\pi(g)\!\otimes\!\rho_{\mathrm{in}}(g)^{-1}\big)\;d\lambda(g).
    \end{align}
    \end{minipage}}\\
    In particular, every equivariant $T$ is block-diagonal across frequencies and acts as the identity on $V_\pi$ and as an intertwiner on the fiber–multiplicity space $V_\pi^{*}\!\otimes\!V$.
\end{theorem}


Hence, the equivariant projection can be computed efficiently in Fourier space. Given a linear map $T$, we (i) compute the Fourier transform of the matrix representation of $T$ to obtain the frequency blocks $\widehat T(\pi,\sigma)$; (ii) zero all off-diagonal blocks, setting $\widehat T(\pi,\sigma)\leftarrow 0$ for $\pi\neq\sigma$; (iii) for each $\pi$, project $\widehat T(\pi,\pi)$ onto $\mathrm{Hom}_G(\pi^{*}\!\otimes\!\rho_{\mathrm{in}},\pi^{*}\!\otimes\!\rho_{\mathrm{out}})$ using the averaging formula for $B_\pi$ above; and (iv) apply the inverse Fourier transform to obtain $P_{\mathrm{equiv}}(T)$ in the spatial domain.

This procedure is illustrated by the commutative diagram in Figure~\ref{fig:commutative_diagrams}, and a corresponding pseudo-code implementation is given in Algorithm~\ref{alg:projection} on the right.

\subsection{Asymptotic Cost}\label{sec:asymp_cost}

We now want to briefly comment on the computational complexity of calculating the projection for both finite and continuous groups.

\paragraph{Finite groups.}
For finite groups we use Equation~\ref{eqn:projection_operator} directly. 
For a linear layer with weights 
\(W \in \mathbb{C}^{d_{\mathrm{out}}\times d_{\mathrm{in}}}\) and 
\(N_\ell = d_{\mathrm{out}} d_{\mathrm{in}}\) parameters, the projection evaluates
\(\pi_{\mathrm{out}}(g)^{\ast} W \pi_{\mathrm{in}}(g)\) for each \(g\in G\), where
\(\pi_{\mathrm{out}}(g)\), \(\pi_{\mathrm{in}}(g)\) are the representation matrices.
Each step costs 
\(O(d_{\mathrm{out}}^{2} d_{\mathrm{in}}) + O(d_{\mathrm{out}} d_{\mathrm{in}}^{2})\),
which is \(O(d_{\mathrm{out}}^{3})\) under \(d_{\mathrm{in}} \sim d_{\mathrm{out}}\).
Since \(N_\ell \sim d_{\mathrm{out}}^{2}\), this is \(O(N_\ell^{3/2})\) per group element, and
\(O\bigl(|G| \, N_\ell^{3/2}\bigr)\).
\paragraph{Continuous groups.}
For continuous groups, we use the Fourier-domain projection. If a model is parameterised spectrally, masking and averaging a weight matrix \(W \in \mathbb{C}^{d_{\mathrm{out}}\times d_{\mathrm{in}}}\) cost \(O(N_\ell)\). We note that this is precisely the regime of steerable CNNs, where kernels are
parameterised directly in such blocks.

\section{Experiments}\label{sec:experiments}

In this section, we conduct three sets of experiments to demonstrate the feasibility and efficiency of our approach to learn (approximate) equivariance from data. For implementation details and information on hyperparameters, see Appendix \ref{app:details}.\footnote{Source code for reproducing the experiments will be released with the camera-ready version.}

\subsection{Example: Learned $SO(2)$ Invariance}\label{sec:so2_equivariance}

\begin{figure*}[t]
    \centering
    \includegraphics[width=1.0\linewidth]{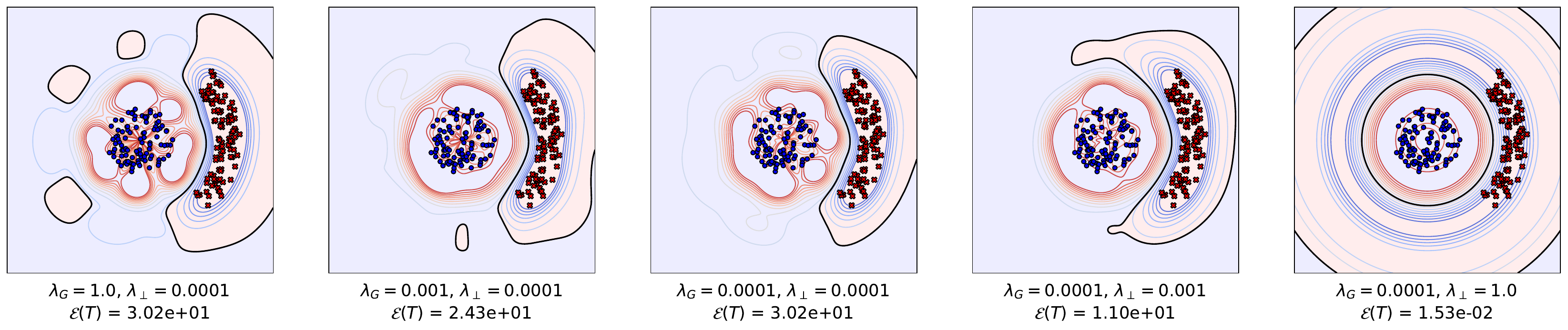}
    \caption{Controlling the degree of learned $SO(2)$ invariance by tuning the parameters $\lambda_G$ and $\lambda_\perp$, which penalise the projections of the equivariant and non-equivariant components, respectively.}
    \label{fig:so2_vary_inv_levels}
\end{figure*}

\begin{figure*}[!t]
    \centering
    \includegraphics[width=1.0\linewidth]{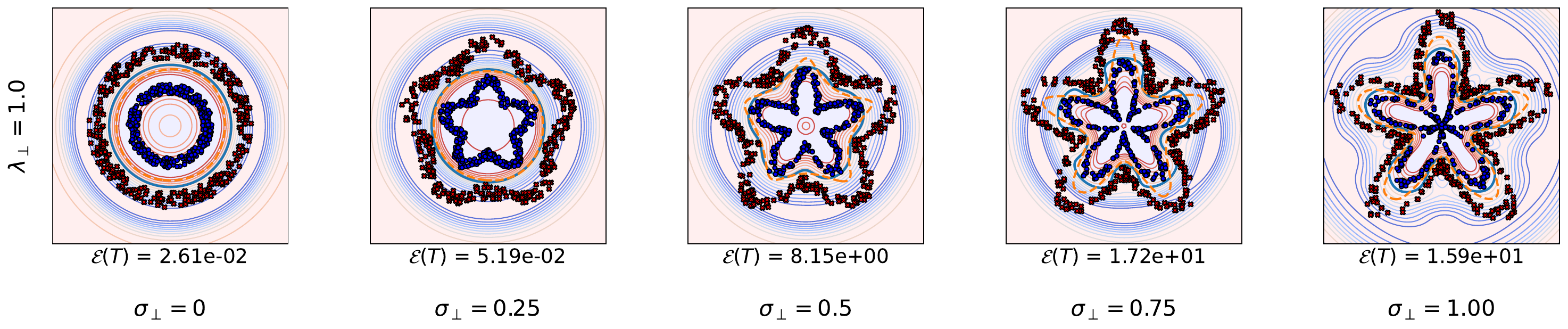}
    \caption{Effect of increasing angular perturbation at fixed projection strength. Each panel shows the decision boundary and level sets of the approximately $\mathrm{SO}(2)$-invariant network (blue) and an MLP (orange) trained with a fixed non-equivariant penalty $\lambda_\perp = 1.0$ on datasets with growing angular ``wave'' amplitude $\sigma_\perp$ (left to right). As $\sigma_\perp$ increases, the decision boundary becomes more angle-dependent and the learned classifier departs from perfect radial symmetry only where required to fit the data, while remaining nearly circular elsewhere. The empirical invariance defect $\mathcal{E}(T)$ for each setting is reported beneath the corresponding panel.}
    \label{fig:wavey_rings_lambda_one_row}
\end{figure*}

We first want to illustrate the approach in Section~\ref{sec:equiv_projection_regularisation} on a simple toy problem (Figure~\ref{fig:so2_vary_inv_levels}). The task is binary classification on two point clouds in $\mathbb{R}^2$. Using polar coordinates $(r,\theta)$, we sample an inner disk-shaped cloud (blue, label $+1$), and the outer angular section of an annulus (red, label $-1$). We then train an approximately $\mathrm{SO}(2)$-invariant MLP with the following structure on this dataset: We first project inputs $(x, y)\in\mathbb{R}^2$ onto circular harmonics up to degree $M$, adding $C$ radial channels via radial embedding functions, to obtain equivariant irreps features
\(
H \;\in\; \mathbb{C}^{(2M+1)\times C_{\text{hid}}}.
\)
We then apply two fully connected complex linear layers
\begin{align}
    L_i:\ \mathbb{C}^{(2M+1)\times C_\text{hid}}\!\to\mathbb{C}^{(2M+1)\times C_{\text{hid}}}, 
\end{align}
followed by an $\mathrm{SO}(2)$-equivariant tensor product. Lastly, we extract the invariant component and pass its real part through a final real-valued linear head $L_{\text{final}}:\ \mathbb{R}^{C_{\text{hid}}}\!\to\mathbb{R}$ to produce the scalar logit. For a more in-depth description of this architecture, see Appendix \ref{app:details_so2}.

In this setting, the projection onto the equivariant subspace reduces to masking. Let $W_i\in\mathbb{C}^{((2M+1)C)\times((2M+1)C)}$ denote the flattened weight matrix of an intermediate linear layer. Define the mask $M\in\mathbb{R}^{((2M+1)C)\times((2M+1)C)}$ by
\[
M_{(m_1,c_1),(m_2,c_2)} \;=\; \delta_{m_1,m_2},
\]
i.e., only blocks with matching harmonic order $m$ are kept. The projected weights are $P(W_i)=M\odot W_i$, where $\odot$ denotes elementwise multiplication. The overall objective is
\[
L \;=\; L_{\text{task}} \;+\; \lambda_G \sum_i \|W_i\| \;+\; \lambda_\perp \sum_i \|W_i - M\odot W_i\|,
\]
with $\lambda_G,\lambda_\perp\ge 0$ and $L_{\text{task}}$ the standard classification loss.

In Figure~\ref{fig:so2_vary_inv_levels}, we compare trained models across different values of $(\lambda_G,\lambda_\perp)$. From left to right, we first reduce $\lambda_G$ and then increase $\lambda_\perp$, enforcing progressively stronger invariance. For a full 2D grid for different combinations of $(\lambda_G,\lambda_\perp)$ see Figure~\ref{fig:2d_grid_vary_equiv} in Appendix~\ref{app:details_so2}. As the regularisation intensifies, the decision boundary becomes increasingly $\mathrm{SO}(2)$-invariant, confirming that the proposed projection-based regulariser effectively pushes the model toward invariance. Consistently, the empirical equivariance defect
\begin{equation}
    \mathcal{E}_\textrm{emp}(T) \;=\; \sum_{k,l}\! \Big\| \rho_{\mathrm{out}}(g_l)\,T(x_k)\;-\;T\big(\rho_{\mathrm{in}}(g_l)\,x_k\big)\Big\|
\end{equation}
with $k$ ranging over data samples and $g_l$ drawn as random rotations in $\mathrm{SO}(2)$, decreases from left to right.

In a second experiment, we probe the behaviour of the regulariser when the target function departs from exact $\mathrm{SO}(2)$-invariance by making the labels increasingly dependent on the polar angle. Starting from two concentric rings, we introduce an angular ``wave'' perturbation of amplitude $\sigma_\perp$ in the radial direction, such that for $\sigma_\perp = 0$ the data distribution is rotationally symmetric, whereas larger $\sigma_\perp$ produce interlocking rings (Figure~\ref{fig:wavey_rings_lambda_one_row} in appendix C). We train the approximately $\mathrm{SO}(2)$-invariant network with projection-based regularisation alongside a plain MLP baseline on these datasets and compare both the learned decision boundaries and the empirical defect $\mathcal{E}_\textrm{emp}(T)$. As $\sigma_\perp$ increases, the regularised model departs from strict invariance only insofar as needed to fit the angularly perturbed rings. This illustrates how the projection penalty (even for constant values of $\lambda_\perp$) furnishes a tunable bias toward invariance that can be gradually traded off against fitting angle-dependent structure in the data. For a full grid, where we also vary the value of $\lambda_\perp$, see Figure~\ref{fig:wavey_rings_lambda_grid} in Appendix~\ref{app:details_so2}.

\begin{table*}[ht]
    \centering
    \scriptsize 
    \caption{Results on three synthetic smoke-plume datasets exhibiting approximate symmetries. {We report means and standard deviations of pixel-wise MSE over 5 random seeds.} \textit{Future} indicates that the test set occurs after the training period; \textit{Domain} indicates that training and test sets come from different spatial regions. Adding our proposed \textit{equivariance regulariser} \textcolor{blue}{(\texttt{+Reg})} consistently improves performance.}
    \resizebox{1.0\textwidth}{!}{%
        \begin{tabular}{c | c | c c c c c c c c c}
            \toprule
            \multicolumn{2}{c}{Model} & \texttt{Conv} & \texttt{Equiv} & \texttt{Rpp} & \texttt{CLCNN} & \texttt{Lift} & \texttt{RGroup} & \texttt{+Reg} & \texttt{RSteer} & \texttt{+Reg} \\
            \midrule
            \multirow{2}{*}{Translation} 
            & Future &  --- & $0.94 {\pm \scriptstyle 0.02}$ & $0.92 {\pm \scriptstyle 0.01}$ & $0.92 {\pm \scriptstyle 0.01}$ & $0.87 {\pm \scriptstyle 0.03}$ & $\mathbf{0.71 {\pm \scriptstyle 0.01}}$ &  $\mathbf{0.72 {\pm \scriptstyle 0.01}}$ & --- & --- \\
            & Domain & --- & $0.68 {\pm \scriptstyle 0.05}$ & $0.93 {\pm \scriptstyle 0.01}$ & $0.89 {\pm \scriptstyle 0.01}$ & $0.70 {\pm \scriptstyle 0.00}$ & $\mathbf{0.62 {\pm \scriptstyle 0.02}}$ & $\mathbf{0.62 {\pm \scriptstyle 0.01}}$ & --- &  --- \\
            \cmidrule(lr){1-11}
            \multirow{2}{*}{Rotation} 
            & Future & $1.21 {\pm \scriptstyle 0.01}$ & $1.05 {\pm \scriptstyle 0.06}$ & $0.96 {\pm \scriptstyle 0.10}$ & $0.96 {\pm \scriptstyle 0.05}$ & $0.82 {\pm \scriptstyle 0.08}$ & $0.82 {\pm \scriptstyle 0.01}$ & $0.80 {\pm \scriptstyle 0.01}$ &  $0.80 {\pm \scriptstyle 0.00}$ & $\mathbf{0.79 {\pm \scriptstyle 0.00}}$ \\
            & Domain & $1.10 {\pm \scriptstyle 0.05}$ & $0.76 {\pm \scriptstyle 0.02}$ & $0.83 {\pm \scriptstyle 0.01}$ & $0.84 {\pm \scriptstyle 0.10}$ & $0.68 {\pm \scriptstyle 0.09}$ & $0.73 {\pm \scriptstyle 0.02}$ & $0.67 {\pm \scriptstyle 0.01}$  & $0.67 {\pm \scriptstyle 0.01}$ & $\mathbf{0.58 {\pm \scriptstyle 0.00}}$ \\
            \cmidrule(lr){1-11}
            \multirow{2}{*}{Scaling} 
            & Future & $0.83 {\pm \scriptstyle 0.01}$ & $0.75 {\pm \scriptstyle 0.03}$ & $0.81 {\pm \scriptstyle 0.09}$ & $1.03 {\pm \scriptstyle 0.01}$ & $0.85 {\pm \scriptstyle 0.01}$ & $0.80 {\pm \scriptstyle 0.01}$ & $0.81 {\pm \scriptstyle 0.00}$ & $0.70 {\pm \scriptstyle 0.01}$ & $\mathbf{0.62 {\pm \scriptstyle 0.01}}$\\
            & Domain & $0.95 {\pm \scriptstyle 0.02}$ & $0.87 {\pm \scriptstyle 0.02}$ & $0.86 {\pm \scriptstyle 0.05}$ & $0.83 {\pm \scriptstyle 0.05}$ & $0.77 {\pm \scriptstyle 0.02}$ & $0.88 {\pm \scriptstyle 0.01}$ & $0.88 {\pm \scriptstyle 0.02}$ & $0.73 {\pm \scriptstyle 0.01}$ & $\mathbf{0.69 {\pm \scriptstyle 0.01}}$\\
            \bottomrule
        \end{tabular}
    }
\end{table*}

\subsection{Imperfectly Symmetric Dynamical Systems}

In this section, we follow the experimental design of \citet{Wang2022} and evaluate our regulariser when applied to their relaxed group and steerable convolutional layers. Using \texttt{PhiFlow} \citep{holl2024phiflow}, we generate $64\times64$ two-dimensional smoke advection–diffusion simulations with varied initial conditions under relaxed symmetries. Each network is trained to predict the velocity field one step ahead.
To test generalisation, we consider two out-of-distribution settings. In the \emph{Future} setting, models predict velocity fields at time steps that are absent from the training distribution, while remaining within spatial regions that were seen during training. In the \emph{Domain} setting, we evaluate at the same time indices as training but at spatial locations that were not seen. The data are produced to break specific symmetries in a controlled way: for \emph{translation}, we generate series for $35$ distinct inflow positions and split the domain horizontally into two subdomains with different buoyancy forces so that plumes diffuse at different rates across the interface; for \emph{discrete rotation}, we simulate $40$ combinations of inflow position and buoyancy, where the inflow pattern alone is symmetric under $90^\circ$ rotations about the domain centre but a position-dependent buoyancy factor breaks rotational equivariance; and for \emph{scaling}, we run $40$ simulations with different time steps $\Delta t$ and spatial resolutions $\Delta x$ to disrupt scale equivariance.%
We compare the relaxed group convolutional networks (\texttt{RGroup}) and relaxed steerable CNNs (\texttt{RSteer}) introduced by \citet{Wang2022} with several baselines: a standard CNN (\texttt{Conv}), an equivariant convolutional network (\texttt{Equiv}) \citep{weiler2019general, sosnovikscale}, Residual Pathway Priors (\texttt{RPP}) \citep{finzi2021residual}, a locally connected network with an explicit equivariance penalty in the loss (\texttt{CLNN}) and \texttt{Lift} \citep{wang2022equivariant}. We indicate the addition of our regulariser with the suffix \texttt{+Reg}.

Across these settings, incorporating our regulariser preserves performance when approximate translation equivariance holds and delivers substantial improvements in the rotation and scaling regimes. In short, the penalty promotes the desired approximate equivariance where symmetry is only partially present, without degrading accuracy where the symmetry is already well aligned with the data.



\begin{table}[ht]
    \centering
    \scriptsize
    \setlength{\tabcolsep}{0.5pt}
    \renewcommand{\arraystretch}{1.05}
    \caption{{CT metal artefact reduction on AAPM. We report PSNR/SSIM, training throughput (batch sizes $4$/$12$ where stable), inference throughput, epoch time, and peak memory; sample-based regularisation is limited to $\leq 4$, while baselines and ours scale to $12$}. Baselines: $^1$\citet{wang2022acdnet}, $^2$\citet{Bai_2025_CVPR}, $^3$\citet{finzi2021residual}, $^4$\citet{wang2021dicdnet}, $^5$ \citet{wang2023oscnet}}
    \begin{tabular}{lccccccc}
        \toprule
        Model &
        \#params &
        \multicolumn{2}{c}{Throughput} &
        Epoch&
        Mem&
        \multicolumn{2}{c}{AAPM} \\
         &
         &
        \multicolumn{2}{c}{(no./GPU-s)} &
        &
         &
        \multicolumn{2}{c}{} \\
        \cmidrule(lr){3-4}\cmidrule(lr){7-8}
        & &
        Train $\uparrow$ &
        Inf. $\uparrow$ &
        time (s) $\downarrow$ &
        (GB) &
        PSNR $\uparrow$ &
        SSIM $\uparrow$ \\
        \midrule
        \texttt{ACDNet}$^1$          & 4.2M & 4.90/5.16 & 8.40  & 1108  & 11.08  & 42.08          & 0.9559 \\
        \enspace + sample-based$^2$      & 4.2M & 2.54/-- & 8.38 & 2011 & 21.99 & 40.02          & \textbf{0.9623} \\
        \enspace + test-time projection                    & 4.2M & --        & --    & --   & --    & 23.63          & 0.8384 \\
        \enspace + RPP$^3$           & 6.9M & 3.49/4.14 & 5.37  & 1455 & 11.15 & 37.12            & 0.9413      \\
        \enspace + projection-based (ours)                 & 4.2M & 4.25/4.99 & 7.44  & 1202   & 11.11 & \textbf{42.68} & 0.9620 \\
        \midrule
        \texttt{DICDNet}$^4$        & 4.3M & 8.38/9.72 & 11.86    & 632   & 10.90    & 41.44          & 0.9468 \\
        \enspace + sample-based$^2$      & 4.3M & 4.05/-- & 10.15 & 1303 & 23.93 & 41.47          & 0.9464 \\
        \enspace + test-time projection                    & 4.3M & --        & --    & --   & --    & 41.59          & 0.9602 \\
        \enspace + RPP$^3$           & 6.6M & 3.10/6.10 & 6.93  & 1028 & 12.08 & 39.42             & 0.9481      \\
        \enspace + projection-based (ours)                 & 4.3M & 5.77/7.82 & 10.11 & 782  & 12.05 & \textbf{41.52} & \textbf{0.9605} \\
        \midrule
        \texttt{OSCNet}$^5$          & 4.3M & 8.59/9.86 & 12.00 & 624  & 10.37    & \textbf{42.36} & 0.9596 \\
        \enspace + sample-based$^2$      & 4.3M & 4.05/-- & 10.13 & 1304 & 23.93 & 41.50          & 0.9593 \\
        \enspace + test-time projection                    & 4.3M & --        & --    & --   & --    & 41.37          & 0.9609 \\
        \enspace + RPP$^3$           & 6.6M & 4.51/6.14   & 6.92  & 1016 & 12.08 & 39.45             & 0.9507      \\
        \enspace + projection-based (ours)                 & 4.3M & 5.66/7.87   & 10.14 & 769   & 12.05 & 41.88          & \textbf{0.9612} \\
        \bottomrule
    \end{tabular}
    \label{tab:ct_mar}
\end{table}

\subsection{CT-Scan Metal Artifact Reduction}\label{sec:experiments_discrete_rotation_equivariance}

We compare our approach with a sample-based equivariance penalty on metal artefact reduction (MAR) for CT scans. Metal implants introduce characteristic streaking artefacts that obscure clinically relevant structures. The task is to map a corrupted slice to its artefact-reduced counterpart.
We use the AAPM CT-MAR Grand Challenge datasets \citep{aapm_ct_mat, aapm_ct_mat_tool}, comprising $14{,}000$ head and body CT slices with synthetic metal artefacts (Table~\ref{tab:ct_mar} and Appendix~\ref{app:ct_mar}, Figure~\ref{fig:ct_mar} for a visual comparison). The datasets were generated with the open-source CT simulation environment XCIST \citep{wu2022xcist}, using a hybrid data-simulation framework that combines publicly available clinical images \citep{yan2018deepLesion, goren2017uclh} and virtual metal objects.
Following \citet{Bai_2025_CVPR}, we adapt three convolution-based architectures \texttt{ACDNet} \citep{wang2022acdnet}, \texttt{DICDNet} \citep{wang2021dicdnet}  and \texttt{OSCNet} \citep{wang2023oscnet} by encouraging rotation equivariance with respect to the discrete group $C_4$ (rotations by multiples of $90^\circ$). We compare the unregularised baselines, the sample-based regulariser of \citet{Bai_2025_CVPR}, and the same networks equipped with our projection-based regulariser. {Additionally, we compare with Residual Pathway Priors (RPPs)~\citep{finzi2021residual} and a train-then-project variant, in which we first train a non-equivariant model and then project its linear layers onto the equivariant subspace at test time using our projection operator.}

For steerable CNN layers whose channels are organised into orientation groups of four, the layer-wise projection acting on a kernel
$K \in \mathbb{R}^{C'_{\mathrm{out}}\times C'_{\mathrm{in}}\times 4 \times 4 \times s \times s}$ is
\begin{equation}
    \label{eq:c4_group_average_main_text}
    P_{\mathrm{equiv}}(K) \;=\; \tfrac14 \sum_{r=0}^{3} S^{r}\, \big(\mathrm{rot}_{r} K\big)\, S^{-r},
\end{equation}
where $S$ is the $4\times 4$ cyclic-shift matrix on orientation channels and $\mathrm{rot}_r$ rotates the spatial kernel by $90^\circ r$. For a derivation of this expression, see Appendix~\ref{app:c4_steerable_projection}.

In contrast, \citet{Bai_2025_CVPR} penalise a term that samples both a data point and a group element. For each sample $x$ they draw a random $r\in C_4$ and add
\begin{equation}
    \label{eq:c4_sample_penalty_inst}
    L_{\mathrm{equiv}}(x, r) \;=\; 
    \big\|
    S^r\,\mathrm{rot}_r\,K(x)\;-\;K\!\big(S^r\,\mathrm{rot}_r x\big)
    \big\|^2
\end{equation}
to the task loss. {This requires an extra forward pass for each sampled rotation and each data sample, with asymptotic cost $O(N_{\textrm{samples}} \cdot \textrm{cost}_\textrm{forward})$ where $N_\textrm{samples}$ is the number of sampled group elements and $\textrm{cost}_\textrm{forward}$ is the cost of a single forward pass. By contrast, as derived in Section~\ref{sec:asymp_cost}, our projection-based regulariser $\|P_\textrm{equiv}(\cdot)\|$ incurs a cost that is linear in the number of parameters, does not sample rotations or data, introduces no extra forward passes, and has zero estimator variance.}

In the fixed-batch setting, we use batch size $4$ for all methods. In the max-feasible setting, the sample-based regulariser remains at batch size $4$ (limited by the extra forward/activation memory), whereas the baselines and our projection-based regulariser scale to batch size $12$ due to unchanged per-sample compute and memory. Our projection-based regulariser delivers competitive or superior reconstruction quality, surpassing the sample-based penalty in all metrics across all settings but one, and improving over the unregularised baselines in most cases. Owing to the extra forward pass in Equation \ref{eq:c4_sample_penalty_inst}, the sample-based approach is constrained to smaller batch sizes and lower throughput. Even under the fixed-batch protocol, its throughput is $42$–$47\%$ lower than ours; under the max-feasible protocol, the gap widens to $54$–$61\%$. These results indicate that projection-based regularisation achieves stronger $C_4$-equivariance with better hardware efficiency by avoiding per-sample group sampling. {Similarly, due to the overparameterisation of the equivariant subspace, RPPs incur slower runtime during both training and inference and require more learnable parameters, and still underperform our approach in reconstruction quality.}

\subsection{MedMNIST}\label{sec:medmnist}

To evaluate our projection-based regulariser for \emph{3D} steerable CNNs under \emph{continuous} approximate symmetries, we follow the MedMNIST v2 \citep{medmnistv2} experiments of \citet{veefkind24}. We consider the 3D classification tasks \texttt{Nodule}, \texttt{Synapse}, and \texttt{Organ}, which exhibit different degrees of (approximate) rotational symmetry.
We compare a standard 3D CNN (\texttt{CNN}) to steerable CNNs (\texttt{SCNN}s) equivariant to $SO(3)$ and $O(3)$ \citep{weiler20183dsteerablecnnslearning}, \texttt{RPP}s \citep{finzi2021residual}, and partially equivariant steerable CNNs (\texttt{P-SCNN}) from \citet{veefkind24}. To apply our method, we use the ESCNN library \citep{cesa2022a} to decompose each 3D convolutional kernel into equivariant and non-equivariant components, and regularise the two parts separately to encourage approximate $SO(3)$- or $O(3)$-equivariance.
\textcolor{blue}{
Table~\ref{tab:medmnist} reports test accuracy over $3$ seeds, and Table~\ref{tab:model_efficiency} reports the corresponding parameter counts and runtime metrics.
}

For \texttt{Nodule}, most methods achieve comparable accuracy, with \texttt{RPP} as the main exception. In \texttt{Synapse}, (partial) rotational symmetry appears to be a helpful inductive bias: fully equivariant \texttt{SCNN}s improve over the \texttt{CNN}, while the best results are achieved by \texttt{P-SCNN} and the proposed projection-based method. By contrast, \texttt{Organ} seems to penalise strict equivariance: fully equivariant models underperform non-invariant alternatives, consistent with equivariance preventing left-right discrimination (see Appendix~\ref{app:confusion}). This asymmetry is further supported by the sensitivity study in Appendix~\ref{app:sensitivity_lambdas_(s)o3}, where smaller values of $\lambda_\perp$ show to perform best on \texttt{Organ}, whereas \texttt{Nodule} and \texttt{Synapse} benefit from larger $\lambda_\perp$. The strongest performance here comes from the proposed method and \texttt{RPP}. 
\textcolor{blue}{
At the same time, we note that our proposed method achieves higher training and inference throughput than the other considered models, and also attains the lowest epoch time. 
}

\textcolor{blue}{
Taken together, the results indicate that the proposed projection-based regulariser provides a favourable accuracy--efficiency trade-off. It remains robust across varying symmetry levels, retains the benefits of approximate equivariance, and avoids the substantial overhead of partial steerable CNNs and residual pathway priors.
}

\begin{table}[t]
\centering
\scriptsize
\setlength{\tabcolsep}{3.8pt}
\renewcommand{\arraystretch}{1.05}
\caption{Test accuracies on MedMNIST datasets. Results for baselines (CNN, RPP, (P)-SCCN) are from \citet{veefkind24}.}
\label{tab:medmnist}
\begin{tabular}{ll|ccc}
\toprule
\makecell{\textbf{Network}\\\textbf{Group}} &
\makecell{\textbf{Partial}\\\textbf{Equivariance}} &
\textbf{Nodule} & \textbf{Synapse} & \textbf{Organ} \\
\midrule\midrule
\texttt{CNN} & N/A &
$\underline{0.873} \pm \scriptstyle 0.005$ &
$\underline{0.716} \pm \scriptstyle 0.008$ &
$\underline{0.920} \pm \scriptstyle 0.003$ \\
\midrule
\multirow{4}{*}{$SO(3)$} & \texttt{None} &
$\underline{0.873} \pm \scriptstyle 0.002$ &
$0.738 \pm \scriptstyle 0.009$ &
$0.607 \pm \scriptstyle 0.006$ \\
& \texttt{RPP} &
$0.801 \pm \scriptstyle 0.003$ &
$0.695 \pm \scriptstyle 0.037$ &
$\underline{0.936} \pm \scriptstyle 0.002$ \\
& \texttt{P-SCNN} &
$0.871 \pm \scriptstyle 0.001$ &
${0.770} \pm \scriptstyle 0.030$ &
$0.902 \pm \scriptstyle 0.006$ \\
& \texttt{ours} &
$0.862 \pm \scriptstyle 0.018$ &
$\underline{\textbf{0.774}} \pm \scriptstyle 0.025$ &
$0.933 \pm \scriptstyle 0.000$ \\
\midrule
\multirow{4}{*}{$O(3)$} & \texttt{None} &
$0.868 \pm \scriptstyle 0.009$ &
$0.743 \pm \scriptstyle 0.004$ &
$0.592 \pm \scriptstyle 0.008$ \\
& \texttt{RPP} &
$0.810 \pm \scriptstyle 0.013$ &
$0.722 \pm \scriptstyle 0.023$ &
$0.940 \pm \scriptstyle 0.006$ \\
& \texttt{P-SCNN} &
$\underline{\textbf{0.873}} \pm \scriptstyle 0.008$ &
$0.769 \pm \scriptstyle 0.005$ &
$0.905 \pm \scriptstyle 0.004$ \\
& \texttt{ours} &
$0.860 \pm \scriptstyle 0.020$ &
$\underline{{0.773}} \pm \scriptstyle 0.025$ &
$\underline{\textbf{0.943}} \pm \scriptstyle 0.005$ \\
\bottomrule
\end{tabular}
\vspace{-0.3cm}
\end{table}


\begin{table}[t]
\centering
\scriptsize
\setlength{\tabcolsep}{4pt}
\renewcommand{\arraystretch}{0.95}
\newcolumntype{C}{>{\centering\arraybackslash}p{1.45cm}}
\caption{Computational efficiency on MedMNIST. We report the number of parameters, training and inference throughput, and epoch time.}
\label{tab:model_efficiency}
\resizebox{\linewidth}{!}{%
\begin{tabular}{l|c|C C|c}
\toprule
\multirow{2}{*}{\textbf{Model}} &
\multirow{2}{*}{\textbf{\#params}} &
\multicolumn{2}{c|}{\textbf{Throughput (samples/s)} $\uparrow$} &
\multirow{2}{*}{\makecell{\textbf{Epoch}\\\textbf{time (s)} $\downarrow$}} \\[-0.15em]
\cmidrule(lr){3-4}
& & \textbf{Train} & \textbf{Inference} & \\[-0.15em]
\midrule
\texttt{CNN}        & 18.1M & 365.3 & 897.1  & 3.2  \\
\midrule
\texttt{SO(3)-SCNN} & 0.1M  & 196.7 & 666.7  & 5.9  \\
\texttt{P-SCNN}     & 11.1M & 82.4  & 665.9  & 14.1 \\
\texttt{RPP}        & 18.3M & 130.7 & 391.6  & 9.1  \\
\midrule
\texttt{ours}       & 2.0M  & \textbf{638.2} & \textbf{5606.5} & \textbf{1.8} \\
\bottomrule
\end{tabular}%
}
\vspace{-0.3cm}
\end{table}

\section{Conclusion}

In this work, we introduced projection-based regularisation - a theoretically grounded approach to learned equivariance which directly penalises model weights and regularises over the entire group instead of only point-wise, per-sample regularisation. For operators for which no closed-form solution of the projection can be computed efficiently in the spatial domain, we provide a general framework for computing the projection efficiently in Fourier space by masking. The experiments demonstrate that across synthetic and real-world experiments, covering both finite and continuous symmetry groups, the proposed approach improves both task performance and runtime.

\paragraph{Limitations and future work.}
\textcolor{blue}{The presented approach is ``architecture-agnostic'', in the sense that the general recipe of using the projection as a regulariser toward equivariance can be applied to any model. As show in Section~\ref{sec:equiv_projection_regularisation}, deriving this projection in closed form is straightforward for linear operators. In future work, we aim to study the projections of non-linearities and architecture-specific modules and use these to build architectures, which admit simple and efficient projections.}
Also, current experiments only evaluate the proposed method for relatively simple groups. In future work, we plan to extend this approach to more complex groups, for example with applications in material sciences.

%

\subsection{Impact Statement}

This paper presents work aimed at advancing the field of learning under symmetries. There are many potential societal consequences of our work, none of which we feel must be specifically highlighted here.

\bibliography{approx_equiv}
\bibliographystyle{icml2026}

\newpage
\appendix
\onecolumn

\section{Projection in the Invariant Case}\label{app:invariance_projection}

In this section, we show that an invariant function $f\in L^2(G)$ only has trivial non-zero Fourier coefficients. 

\begin{lemma}\label{lemma:inv_fourier}
    Let $f\in L^2(G)$ be left invariant with respect to the regular representation $\tau$, i.e. $f(hg)=f(g)$ for all $h,g\in G$. Then $\widehat f(\pi)$ is non-zero if and only if $\pi$ is the trivial representation $\mathbf1: g \mapsto I_{\C}$.
\end{lemma} 
\begin{proof}
    See Appendix \ref{app:proof_lemma_inv}.
\end{proof}

\begin{corollary}\label{cor:projection_invariant}
    Let $f\in L^2(G)$ be any function on $G$ and set $P_\textrm{inv}$ to be the invariant projection. Then $\widehat{P_\text{inv}(f)}(\pi) = \widehat f(\pi) \delta_{\pi, \mathbf1}$.
\end{corollary}
\begin{proof}
    See Appendix \ref{app:proof_corollary_projection_invariant}
\end{proof}

In Figure \ref{fig:commutative_diagrams_invariance}, we schematically depict how we can exploit the simple structure of the projection in the spectral domain $\widehat{P_\text{inv}}$ to efficiently calculate the smoothing operator $P_\text{inv}$.

\begin{figure}[t]
  \centering
    \begin{adjustbox}{max width=\linewidth}
      \begin{tikzcd}[
        row sep=3.8em,
        column sep=6em,
        cells={nodes={font=\small}}
      ]
        f \in L^2(G)
          \arrow[r, "P_\text{inv}"]
          \arrow[d, shift right=1.0ex, "FT"']
        & f_{\text{inv}} \in L^2(G)
          \arrow[d, shift right=1.0ex, "FT"'] \\
        \widehat{f}: \widehat G \to \bigcup_\pi U(V_\pi)
          \arrow[u, shift right=1.0ex, "IFT"']
          \arrow[r, "\widehat{P_\text{inv}}=\delta_{\pi, \mathbf{1}}"']
        & \widehat{f}_{\text{inv}} = \widehat{f_{\text{inv}}}:\widehat G \to \bigcup_\pi U(V_\pi)
          \arrow[u, shift right=1.0ex, "IFT"']
        \arrow[ul, phantom, "\scalebox{1.2}{$\circlearrowleft$}" description]
      \end{tikzcd}
    \end{adjustbox}
    \caption{A commutative diagrams showing how to apply the projection operator in Fourier space. fro invariance, we keep only the trivial representation and discards all other frequencies.}
    \label{fig:commutative_diagrams_invariance}
\end{figure}

\section{Proofs in Section \ref{sec:equiv_projection_regularisation}}

\subsection{Proof of Lemma \ref{lem:proj-defect}}\label{app:proof_lemma_proj_defect}

\begin{proof}[Proof of Lemma \ref{lem:proj-defect}]
    By definition of $P$,
    \begin{equation}
        T - P(T) 
        = \int_G \pi_{\mathrm{out}}(g)^{*}\big(\pi_{\mathrm{out}}(g) \circ T - T\circ \pi_{\mathrm{in}}(g)\big)\, d\lambda(g)
        = \int_G \pi_{\mathrm{out}}(g)^{*}\,\Delta_g(T)\, d\lambda(g).
    \end{equation}
    Pre-/post-composition with the unitaries $\pi_{\mathrm{out}}(g)^{*}$ preserves the Hilert-Schmidt norm, and the norm of an average is at most the average of the norm. Hence
    \begin{equation}
        \|T - P(T)\|_\mu
        \;\le\; \int_G \|\Delta_g(T)\|_\mu\, d\lambda(g)
        \;\le\; \sup_{g\in G}\|\Delta_g(T)\|_\mu
        \;=\; \E(T),
    \end{equation}
    giving the lower bound. For the upper bound, note that $P(T)$ is $G$-equivariant, and therefore,
    \begin{equation}
        \Delta_g(T)
        = \pi_{\mathrm{out}}(g)\big(T - P(T)\big)
          - \big(T - P(T)\big)\pi_{\mathrm{in}}(g).
    \end{equation}
    Taking norms and using that $\pi_{\mathrm{in/out}}(g)$ are unitaries,
    \begin{equation}
        \|\Delta_g(T)\|_\mu \;\le\; \|T - P(T)\|_\mu + \|T - P(T)\|_\mu \;=\; 2\,\|T - P(T)\|_\mu.
    \end{equation}
    Finally, take the supremum over $g\in G$ to obtain $\E(T)\le 2\|T - P(T)\|_\mu$.
\end{proof}

\subsection{Proof of Lemma \ref{lem:defect_composition}}\label{app:proof_lemma_defect_composition}

\begin{proof}[Proof of Lemma \ref{lem:defect_composition}]
    For any composable maps $A,B$, the equivariance defect satisfies the chain rule
    \begin{equation}
        \Delta_g(A\circ B) \;=\; (\Delta_g A)\circ B \;+\; A\circ (\Delta_g B).
    \end{equation}
    Applying this repeatedly to $f_k\circ\cdots\circ f_1$ yields the telescoping identity
    \begin{equation}
        \Delta_g(T)
        = \sum_{i=1}^k \big(f_k\circ\cdots\circ f_{i+1}\big)\circ \Delta_g(f_i)\circ \big(f_{i-1}\circ\cdots\circ f_1\big).
    \end{equation}
    Taking norms and using $\|f_{i+1}\circ \Delta_g(f_i)\|_\mu = \|f_{i+1}\circ (\pi_{out}(g) \circ f_i - f_{i} \circ \pi_{in}(g))\|_\mu\le \mathrm{Lip}(f_{i+1})\,\|\Delta_g(f_i)\|_\mu$ iteratively together with
    $\mathrm{Lip}(f_j)=L_j$ to obtain
    \begin{equation}
        \|\Delta_g(T)\|_\mu
        \;\le\; \sum_{i=1}^k \Bigg(\prod_{m=i+1}^{k} L_m\Bigg)\, \|\Delta_g(f_i)\|_\mu\, \Bigg(\prod_{m=1}^{i-1} L_m\Bigg).
    \end{equation}
    Finally, take $\sup_{g\in G}$ on both sides and note that $\E(T)=\sup_g\|\Delta_g(T)\|$ and
    $\E(f_i)=\sup_g\|\Delta_g(f_i)\|$ to obtain the stated bound.
\end{proof}

\subsection{Proof of Corollary \ref{cor:bound_neural_net}}\label{app:proof_cor_bound_neural_net}

\begin{proof}[Proof of Corollary \ref{cor:bound_neural_net}]
    \textcolor{blue}{To fit into the framework of Lemma \ref{lem:defect_composition}, we choose}
    \[
        f_{2k-1} \coloneqq W^{(k)}, \qquad
        f_{2k} \coloneqq \sigma_k,
    \]
    so that 
    \[
        \quad L_{2k-1} \coloneqq \|W^{(k)}\|_F,\quad L_{2k} \coloneqq \textrm{Lip}(\sigma_k),
    \]
    for $k = 1,\dots,2S-1$. By construction $E(\sigma_k) = 0$ for all $k$, hence $E(f_{2k}) = 0$. Plugging this into Equation~\ref{eqn:decomposition} and noting that the even indices do not contribute, we obtain
    \begin{equation}
    \label{eq:sum-over-layers}
    E(T)
    \;\le\;
    \sum_{k=1}^S
    \Bigg( \prod_{m \neq 2k-1} L_m \Bigg) E\big(W^{(k)}\big).
    \end{equation}
    Now note that the product over $m\neq 2k-1$ contains
    \begin{itemize}
        \item all activation Lipschitz constants $L_{2j} = \textrm{Lip}(\sigma_j)$, $j=1,\dots,2S-1$,
        \item all weight norms $L_{2r-1} = \|W^{(r)}\|$ with $r\neq k$.
    \end{itemize}
    Thus
    \[
    \prod_{m \neq 2k-1} L_m
    =
    \Bigg( \prod_{j=1}^{S-1} \textrm{Lip}(\sigma_j) \Bigg)
    \Bigg( \prod_{\substack{r=1 \\ r\neq k}}^S \|W^{(r)}\| \Bigg),
    \]
    and Equation~\ref{eq:sum-over-layers} becomes
    \begin{equation}
    \label{eq:bound-with-EW}
    E(T)
    \;\le\;
    \Bigg( \prod_{j=1}^{S-1} \textrm{Lip}(\sigma_j) \Bigg)
    \sum_{k=1}^S
    \Bigg( \prod_{\substack{r=1 \\ r\neq k}}^S \|W^{(r)}\| \Bigg)
    E\big(W^{(k)}\big).
    \end{equation}
    
    Next use Lemma~3.2, which states that for each linear layer
    \[
    E\big(W^{(k)}\big)
    \;\le\;
    2 \,\big\| W^{(k)} - P\big(W^{(k)}\big) \big\|_F.
    \]
    Substituting this into Equation~\ref{eq:bound-with-EW} yields
    \begin{equation}
    \label{eq:explicit-C-bound}
    E(T)
    \;\le\;
    2 \Bigg( \prod_{j=1}^{S-1} \textrm{Lip}(\sigma_j) \Bigg)
    \sum_{k=1}^S
    \Bigg( \prod_{\substack{r=1 \\ r\neq k}}^S \|W^{(r)}\|_F \Bigg)
    \big\| W^{(k)} - P\big(W^{(k)}\big) \big\|_F.
    \end{equation}
    
    Define
    \begin{equation}
    \label{eq:def-C}
    C
    :=
    2 \Bigg( \prod_{j=1}^{S-1} \textrm{Lip}(\sigma_j) \Bigg)
    \max_{1 \le k \le S}
    \prod_{\substack{r=1 \\ r\neq k}}^S \|W^{(r)}\|_F.
    \end{equation}
    Then, for every $k$,
    \[
    2 \Bigg( \prod_{j=1}^{S-1} \textrm{Lip}(\sigma_j) \Bigg)
    \prod_{\substack{r=1 \\ r\neq k}}^S \|W^{(r)}\|
    \;\le\;
    C,
    \]
    and Equation~\ref{eq:explicit-C-bound} implies
    \[
    E(T)
    \;\le\;
    C \sum_{k=1}^S \big\| W^{(k)} - P\big(W^{(k)}\big) \big\|_F.
    \]
    
    This is exactly Eq.~(9), with the dependence of $C$ on the norms $\|W^{(k)}\|$ and Lipschitz constants $\textrm{Lip}(\sigma_j)$ made explicit in Equation~\ref{eq:def-C}.
\end{proof}

\subsection{Proof of Lemma \ref{lemma:inv_fourier}}\label{app:proof_lemma_inv}
\begin{proof}[Proof of Lemma \ref{lemma:inv_fourier}]
    We define the invariance operator of a function $f \in L^2(G)$ as 
    \begin{align}
        f_\text{inv}(g) = \int_G f(hg) \, d\lambda(h)
    \end{align}
    The Fourier coefficients of this are
    \begin{align}
        \widehat{f_\text{inv}}(\pi) 
        &= \int_G f_\text{inv}(g) \, \pi(g)^* \, d\lambda(g) \\
    &= \int_G \left( \int_G f(hg) \, d\lambda(h) \right) \pi(g)^* \, d\lambda(g) \\
    &= \int_G f(x) \left( \int_G \pi(h^{-1}x)^* \, d\lambda(h) \right) d\lambda(x) \quad \text{substituting $x=hg$ $\implies g = h^{-1} x$}\\
    &= \int_G f(x) \left( \int_G \pi(h^{-1})^* \, d\lambda(h) \right) \pi(x)^* \, d\lambda(x) \\
    &= \int_G f(x) \left( \int_G \pi(h)^* \, d\lambda(h) \right) \pi(x)^* \, d\lambda(x). \quad \text{invariance of Haar measure}
    \end{align}
    Define $A_\pi := \int_G \pi(h)^* \, d\lambda(h) \in \text{End}(V_\pi)$. Note that $A_\pi$ is $\pi$-equivariant; indeed, for all $g \in G$,
    \begin{align}
        \pi(g)\,A_\pi 
        &=\int_G\pi(g)\,\pi(h)^*\,d\lambda(h) \\
        &=\int_G\pi(gh^{-1})\,d\lambda(h)\\
        &=\int_G\pi(k)^*\pi(g)\,d\lambda(k) \quad \quad \text{substituting $k=ghg^{-1}$ $\implies gh^{-1} = k^{-1} g$}\\
        &= A_\pi \,  \pi(g),
    \end{align}
    Hence by Schur's lemma (since $\pi$ is irreducible), we have 
    \[
    A_\pi\in\mathrm{End}G(V_\pi)
    \;\cong\;
    \{\,\lambda I: \lambda\in\mathbb{C}\}.
    \]
    So $A_\pi = \lambda I$ for some $\lambda \in \mathbb{C}$.
    
    Now,
    \begin{align}
        \mathrm{tr}\,A_\pi
    =\int_G \mathrm{tr}\bigl(\pi(h)^*\bigr)\,d\lambda(h)
    =\int_G \overline{\chi_\pi(h)}\,d\lambda(h)
    \;=\;\overline{\int_G\chi_\pi(h)\,d\lambda(h)}.
    \end{align}
    But the characters $\chi_\pi$ are orthonormal, so denoting the trivial representation $g \mapsto 1$ by $\textbf{1}$, i.e. have $\chi_\textbf{1}(g) = 1$ for all $g$, we have 
    \begin{align}
        \int_G \chi_\pi(g)\,d\lambda(g)
    =\int_G \chi_\pi(g)\,\overline{\chi_{\mathbf1}(g)}\,d\lambda(g)
    =\langle \chi_\pi(g), \chi_{\mathbf1}(g) \rangle_{L^2(G)}
    =\delta_{\pi,\mathbf1}.
    \end{align}
    Finally, this gives
    \begin{align}
        d_\pi \lambda \;=\; \mathrm{tr}\,A_\pi \;=\; \implies \lambda \;=\; \frac{\delta_{\pi,\mathbf1}}{d_\pi}
        =\begin{cases}0,&\pi\neq\mathbf1,\\\frac{1}{d_\pi},&\pi=\mathbf1.\end{cases}
    \end{align}
    Substituting this into the above yields
    \begin{align}
         \widehat{f_\text{inv}}(\pi) = \frac{1}{d_\pi}  \widehat{f}(\pi)  \delta_{\pi,\mathbf1}.
    \end{align}
\end{proof}

\subsection{Proof of Corollary \ref{cor:projection_invariant}}\label{app:proof_corollary_projection_invariant}

\begin{proof}[Proof of Corollary \ref{cor:projection_invariant}]
    Since $P_\text{inv}$ is a projection onto the $G$-invariant subspace, $P_\text{inv}(f)$ is always invariant. Hence, by Lemma \ref{lemma:inv_fourier}, $\widehat{P_\text{inv}(f)}(\pi)$ is zero for all $\pi \neq \mathbf 1$. Now note that by invariance, $P_\text{inv}(f)(g) = c$ for all $g \in G$ for some $c \in \C$. We then calculate 
    \begin{equation}
        \widehat{P_\text{inv}(f)}(\mathbf 1) 
        = \int_G P_\text{inv}(f)(g) \; \mathbf{1}(g)^* \: d\lambda(g) 
        = \int_G P_\text{inv}(f)(g) \; d\lambda(g) 
        = \int_G c \; d\lambda(g) 
        = c.
    \end{equation}
    At the same time
    \begin{equation}
        \widehat f(\mathbf 1) 
        = \int_G \; f(g) \; \mathbf{1}(g)^* d\lambda(g)
        = \int_G \; f(g)  d\lambda(g)
        = c.
    \end{equation}
    which concludes the proof.
\end{proof}

\subsection{Proof of Theorem \ref{thm:equiv_fourier}}\label{app:proof_thm_equiv}

\begin{proof}[Proof of Theorem \ref{thm:equiv_fourier}]
    Recall that by Peter--Weyl we have an isometric $G$-equivariant decomposition
    \[
    L^2(G)\;\cong\;\bigoplus_{\pi\in\widehat G} V_\pi\otimes V_\pi^*,
    \]
    where the left regular representation acts as $\tau(g)\cong\bigoplus_{\pi} \pi(g)\otimes I_{V_\pi^*}$.
    Let $T:L^2(G)\to L^2(G)$ be linear and $\tau$-equivariant, i.e.\ $\tau(g)\circ T=T\circ \tau(g)$ for all $g\in G$.
    Write the block-matrix of $T$ in this decomposition as $\widehat T=\{T_{\pi\sigma}\}_{\pi,\sigma\in\widehat G}$ with
    \[
    T_{\pi\sigma}\;:\; V_\sigma\otimes V_\sigma^*\;\longrightarrow\; V_\pi\otimes V_\pi^*.
    \]
    Equivariance implies, for all $g\in G$,
    \begin{equation}\label{eq:block_intertwiner}
    (\pi(g)\otimes I)\,T_{\pi\sigma}\;=\;T_{\pi\sigma}\,(\sigma(g)\otimes I).
    \end{equation}
    Thus $T_{\pi\sigma}$ is an intertwiner from $\sigma$ to $\pi$ (acting on the first tensor factor). By Schur's lemma,
    $T_{\pi\sigma}=0$ whenever $\pi\not\simeq\sigma$. Hence $\widehat T$ is block-diagonal:
    \[
    \widehat T\;\cong\;\bigoplus_{\pi\in\widehat G} T_{\pi\pi},
    \qquad
    T_{\pi\pi}\in\End(V_\pi\otimes V_\pi^*),
    \]
    and Equation~\ref{eq:block_intertwiner} reduces to
    \begin{equation}\label{eq:commutant_condition}
    (\pi(g)\otimes I)\,T_{\pi\pi}\;=\;T_{\pi\pi}\,(\pi(g)\otimes I)\qquad \forall g\in G.
    \end{equation}
    
    We now identify all endomorphisms satisfying Equation~\ref{eq:commutant_condition}.
    Consider the canonical vector space isomorphism
    \[
    \Phi:\;V_\pi\otimes V_\pi^*\xrightarrow{\;\cong\;}\End(V_\pi),
    \qquad
    \Phi(u\otimes\varphi)(v):=\varphi(v)\,u.
    \]
    A direct check shows that under $\Phi$, the action $\pi(g)\otimes I$ corresponds to left multiplication on $\End(V_\pi)$:
    \[
    \Phi\big((\pi(g)\otimes I)(w)\big)\;=\;\pi(g)\,\Phi(w)\qquad \forall w\in V_\pi\otimes V_\pi^*.
    \]
    Define $\widetilde T_\pi:=\Phi\circ T_{\pi\pi}\circ\Phi^{-1}\in\End(\End(V_\pi))$. Then Equation~\ref{eq:commutant_condition}
    is equivalent to
    \begin{equation}\label{eq:left_equiv_End}
    \widetilde T_\pi(\pi(g)X)\;=\;\pi(g)\widetilde T_\pi(X)\qquad \forall g\in G,\ \forall X\in\End(V_\pi).
    \end{equation}
    Let $B:=\widetilde T_\pi(I_{V_\pi})\in\End(V_\pi)$. For any $g\in G$, applying Equation~\ref{eq:left_equiv_End} to $X=I$ gives
    $\widetilde T_\pi(\pi(g))=\pi(g)B$. By linearity, the same identity holds for all $X$ in the linear span of $\{\pi(g):g\in G\}$.
    
    Since $\pi$ is irreducible, the span of $\{\pi(g)\}$ is all of $\End(V_\pi)$ (Burnside's theorem). Therefore,
    for every $X\in\End(V_\pi)$ we have
    \[
    \widetilde T_\pi(X)\;=\;X\,B,
    \]
    i.e.\ $\widetilde T_\pi$ is right multiplication by $B$. Transporting back through $\Phi$ shows that
    \[
    T_{\pi\pi}\;=\;I_{V_\pi}\otimes B_\pi
    \]
    for some $B_\pi\in\End(V_\pi^*)$ (canonically corresponding to $B$ on the dual space).
    Hence
    \[
    \widehat T\;\cong\;\bigoplus_{\pi\in\widehat G}\big(I_{V_\pi}\otimes B_\pi\big),
    \]
    which proves the claimed block structure.
\end{proof}

\section{Details on Vector-valued Signals}\label{app:vector-valued_signals}

\paragraph{Fourier description.}
Peter–Weyl yields the unitary decomposition
\[
  L^2(G) \;\cong\; \bigoplus_{\pi\in\widehat G} V_\pi\otimes V_\pi^{*},
  \qquad
  L^2(G,V) \;\cong\; \bigoplus_{\pi\in\widehat G} V_\pi\otimes\big(V_\pi^{*}\otimes V\big),
\]
where $G$ acts by $\pi$ on the first tensor factor and trivially on $V_\pi^*$, while the fiber transforms by $\rho$.
Accordingly, any bounded linear map
\(
  T:L^2(G,V_{\mathrm{in}})\!\to\!L^2(G,V_{\mathrm{out}})
\)
admits a block form
\[
  \widehat T \;\cong\; \Big(\widehat T(\pi\!,\sigma)\Big)_{\pi,\sigma\in\widehat G},
  \qquad
  \widehat T(\pi\!,\sigma):
  V_\sigma\!\otimes\!\big(V_\sigma^{*}\!\otimes\!V_{\mathrm{in}}\big)
  \longrightarrow
  V_\pi\!\otimes\!\big(V_\pi^{*}\!\otimes\!V_{\mathrm{out}}\big).
\]
Averaging annihilates all off-diagonal $(\pi\neq\sigma)$ blocks and, on each frequency $\pi$, orthogonally projects $\widehat T(\pi,\pi)$ onto the intertwiner space $\mathrm{Hom}_G\!\big(\,\pi^{*}\!\otimes\!\rho_{\mathrm{in}}\,,\,\pi^{*}\!\otimes\!\rho_{\mathrm{out}}\,\big).$

\begin{theorem}[Theorem \ref{thm:vector_block} restated]
    Let $T:L^2(G,V_{\mathrm{in}})\!\to\!L^2(G,V_{\mathrm{out}})$ be linear. Then
    \begin{align}
      \widehat{P_{\mathrm{equiv}}(T)}
      &\;\cong\;
      \bigoplus_{\pi\in\widehat G}
      \Big(I_{V_\pi}\otimes B_\pi\Big), \\
      B_\pi 
      &\;=\;
      \int_G
      \big(\pi(g)^{*}\!\otimes\!\rho_{\mathrm{out}}(g)\big)\,
      \widehat T(\pi,\pi)\,
      \big(\pi(g)\!\otimes\!\rho_{\mathrm{in}}(g)^{-1}\big)\;d\lambda(g),
    \end{align}
    with $B_\pi\in \mathrm{Hom}_G\!\big(\pi^{*}\!\otimes\!\rho_{\mathrm{in}},\,\pi^{*}\!\otimes\!\rho_{\mathrm{out}}\big)$. In particular, every equivariant $T$ is block-diagonal across frequencies and acts as the identity on $V_\pi$ and as an intertwiner on the fiber–multiplicity space $V_\pi^{*}\!\otimes\!V$.
\end{theorem}

\begin{proof}[Sketch]
Decompose both domain and codomain via Peter–Weyl and write $\widehat T$ in blocks $\widehat T(\pi,\sigma)$. Conjugation by $(\tau\!\otimes\!\rho)$ restricts, on the $(\pi,\pi)$ block, to the representation $\pi^{*}\!\otimes\!\rho_{\mathrm{out}}$ on the codomain multiplicity and $\pi^{*}\!\otimes\!\rho_{\mathrm{in}}$ on the domain multiplicity. Averaging is the orthogonal projection onto the commutant, hence onto $\mathrm{Hom}_G(\pi^{*}\!\otimes\!\rho_{\mathrm{in}},\pi^{*}\!\otimes\!\rho_{\mathrm{out}})$, and kills $\pi\neq\sigma$ by Schur orthogonality. The displayed formula is the explicit Bochner average of that projection.
\end{proof}



\section{Implementation Details}\label{app:details}

In this section, we provide additional information on the implementation details of all of our experiments. 

\subsection{Example: Learned $SO(2)$ invariance}\label{app:details_so2}

\paragraph{Data generation.}
Using polar coordinates $(r,\theta)$, we sample the inner cloud (blue, label $+1$) by drawing $r \sim \mathrm{Unif}[0,1]$ and $\theta \sim \mathrm{Unif}[0,2\pi)$, and the outer cloud (red, label $-1$) by drawing $r \sim \mathrm{Unif}[2.3,3]$ and $\theta \sim \mathrm{Unif}\!\big[-\tfrac{\pi}{4}, \tfrac{\pi}{4}\big)$.

\paragraph{Feature map and network.}
We project inputs $(x, y)\in\mathbb{R}^2$ onto circular harmonics up to degree $M=4$ with $C=4$ radial channels as follows: viewing $(x, y)$ as a complex number $z\in\mathbb{C}$ with $r=|z|$ and $\hat z = z/r$, define radial basis functions
\[
b_n(r) \;=\; \exp\!\Big(-\frac{(r-c_n)^2}{2\sigma^2}\Big),\quad \sigma=0.5,\quad c_n\ \text{uniform in }[0,4],\ n=1,\dots,C.
\]
Form the order-$m$ harmonic features by $h^{(m)}(r,\hat z)=\big(b_n(r)\,\hat z^{\,m}\big)_{n=1}^C$ for $m=-M,\dots,M$, and concatenate across $m$ to obtain the embedding
\[
H \;\in\; \mathbb{C}^{(2M+1)\times C}.
\]
We then apply two fully connected complex linear layers
\[
L_1:\ \mathbb{C}^{(2M+1)\times C}\!\to\mathbb{C}^{(2M+1)\times C_{\text{hid}}},\qquad
L_2:\ \mathbb{C}^{(2M+1)\times C_{\text{hid}}}\!\to\mathbb{C}^{(2M+1)\times C_{\text{hid}}},
\]
followed by an $\mathrm{SO}(2)$-equivariant tensor product:
\[
h'_{m_{\text{out}}} \;=\; \sum_{m_1+m_2=m_{\text{out}}} h_{m_1}\,h_{m_2},
\]
with complex multiplication applied channel-wise. We then extract the invariant component $h'_0$ and pass its real part through a final real-valued linear head $L_{\text{final}}:\ \mathbb{R}^{C_{\text{hid}}}\!\to\mathbb{R}$ to produce the scalar logit.

We then train a new model for each combination of $\lambda_G, \lambda_\perp$ (see Figure \ref{fig:so2_vary_inv_levels}) using the Adam optimiser \cite{kingma2014adam} for $200$ epochs with a learning rate of $0.003$. We use a binary cross-entropy loss as task-specific loss.


\textcolor{blue}{\paragraph{Angular perturbation experiment (Figure~\ref{fig:wavey_rings_lambda_one_row}).}To study the interaction between the projection regulariser and violations of exact $\mathrm{SO}(2)$ symmetry, we construct a family of ``wavey'' ring datasets parameterised by an amplitude $\sigma_\perp \ge 0$. For each $\sigma_\perp$ we independently sample angles $\theta_+,\theta_- \sim \mathrm{Unif}[0,2\pi)$ and define class-conditional radii
\[
r_+(\theta_+) \;=\; r_{\mathrm{in}} + \sigma_\perp \sin(f\theta_+) + \epsilon_{\mathrm{in}},\qquad
r_-(\theta_-) \;=\; r_{\mathrm{out}} + \sigma_\perp \sin(f\theta_-) + \epsilon_{\mathrm{out}},
\]
with $(r_{\mathrm{in}},r_{\mathrm{out}}) = (1.1, 2.2)$, frequency $f=5$ and independent jitters $\epsilon_{\mathrm{in}}\sim \mathrm{Unif}[-b_{\mathrm{in}},b_{\mathrm{in}}]$, $\epsilon_{\mathrm{out}}\sim \mathrm{Unif}[-b_{\mathrm{out}},b_{\mathrm{out}}]$ for $(b_{\mathrm{in}},b_{\mathrm{out}}) = (0.15, 0.22)$. Mapping $(r_\pm,\theta_\pm)$ to Cartesian coordinates yields two noisy rings labelled $+1$ (inner) and $-1$ (outer). In Figure~\ref{fig:wavey_rings_lambda_grid}, we consider $\sigma_\perp \in \{0, 0.5, 0.75, 1.0\}$, sample $350$ points per class, and split the data into $80\%$ training and $20\%$ test. For each $(\sigma_\perp,\lambda_\perp)$ we then train (i) the approximately $SO(2)$-invariant architecture described above (blue lines), and (ii) a plain real-valued MLP on the raw coordinates (orange).}

\textcolor{blue}{We see that even for a fixed value of $\lambda_\perp$, the regulariser allows us to capture different effective levels of invariance as the data depart from rotational symmetry; see, for instance, the row with $\lambda_\perp = 1.0$, where the learned classifier remains nearly invariant for small $\sigma_\perp$ and gradually departs from invariance as the angular modulation strengthens. For strongly broken $\mathrm{SO}(2)$ symmetry (e.g. $\sigma_\perp = 1.0$), the decision boundary remains ``as radially symmetric as possible'': away from the perturbed regions the contours revert to circular rings, and in the region between the two classes, around each arm of the star-shaped pattern, the classifier exhibits consistent behaviour across angles.}

\subsubsection{Sensitivity with respect to $\lambda_G$ and $\lambda_\perp$}\label{app:sensitivity_lambdas_so2}

\textcolor{blue}{We study the sensitivity of our method to the scalar weights $\lambda_G$ and $\lambda_\perp$ through two ablation experiments. First, we repeat the experiment from Section~\ref{sec:so2_equivariance} on approximate $SO(2)$ invariance in 2D for $\lambda_G, \lambda_\perp \in \{0, 0.001, 0.01, 0.1\}$; the resulting decision boundaries are shown in Figure~\ref{fig:2d_grid_vary_equiv}. When the penalty on the orthogonal component dominates (e.g. $\lambda_\perp = 0.1$ and $\lambda_G \in \{0, 0.001, 0.01\}$), the decision boundary becomes essentially rotationally invariant. In the regime $\lambda_\perp \approx \lambda_G$, the regulariser effectively reduces to standard Tikhonov ($\ell_2$) regularisation and no longer induces a geometric inductive bias. For $\lambda_\perp < \lambda_G$, the learned level sets increasingly depend on angular information.}

\begin{figure}
    \centering
    \includegraphics[width=0.75\linewidth]{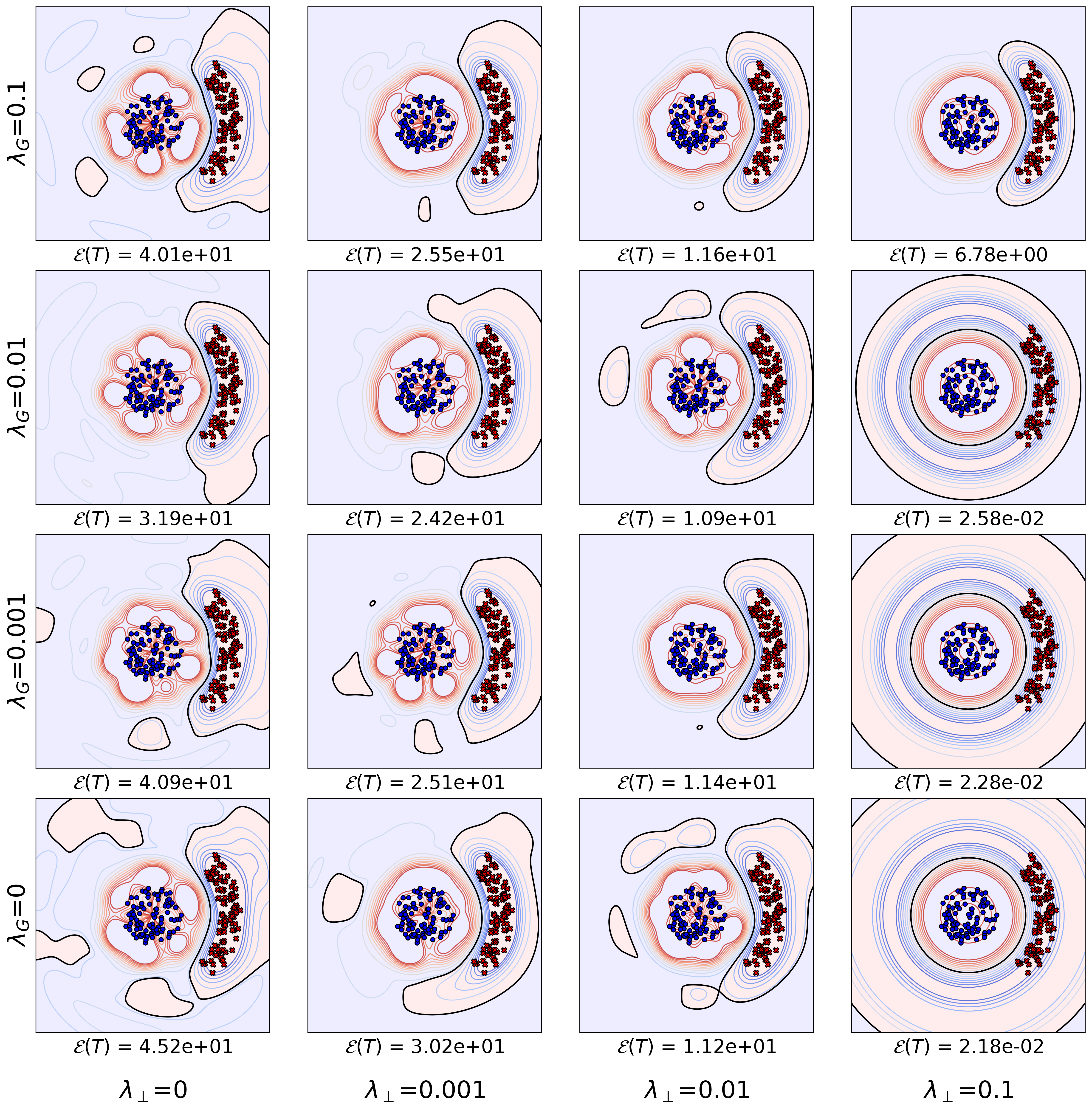}
    \caption{Controlling the degree of learned $SO(2)$} invariance by varying the values of $\lambda_G$ and $\lambda_\perp$ over the grid $\{0, 0.001, 0.01, 0.1\}$.
    \label{fig:2d_grid_vary_equiv}
\end{figure}

\begin{figure}
    \centering
    \includegraphics[width=0.75\linewidth]{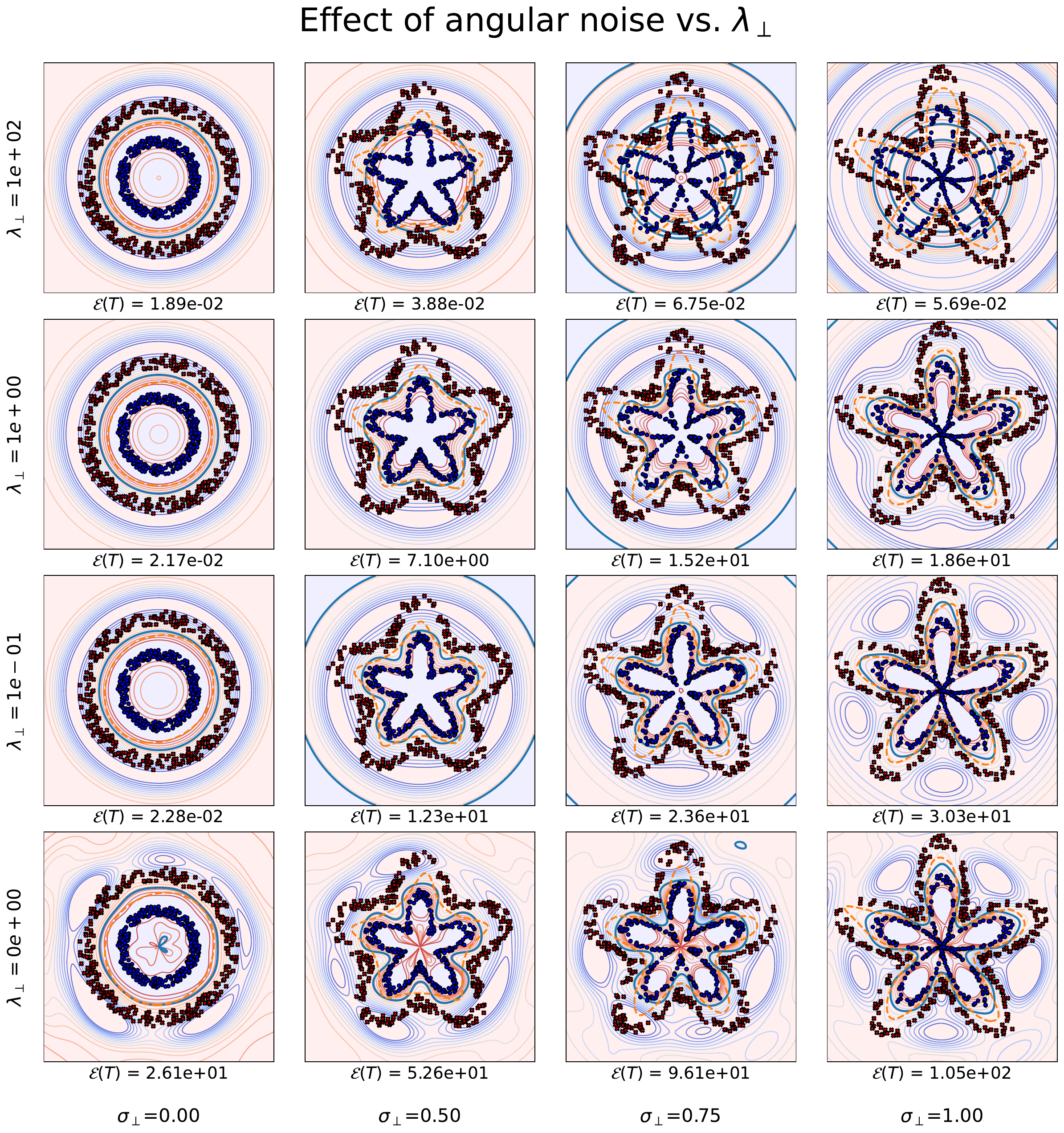}
    \caption{Effect of angular perturbations and projection strength. Columns vary the angular wave amplitude $\sigma_\perp$, rows vary the non-equivariant penalty weight $\lambda_\perp$. Blue contours show level sets of the approximately $\mathrm{SO}(2)$-invariant network and points denote training samples. Orange dashed lines are the decision boundary of a non-equivariant MLP. The value $\mathcal{E}(T)$ underneath each panel is the empirical invariance defect, demonstrating that larger $\lambda_\perp$ keeps the classifier close to invariant even as the Bayes decision boundary becomes increasingly angle-dependent.}
    \label{fig:wavey_rings_lambda_grid}
\end{figure}

\subsection{Imperfectly Symmetric Dynamical Systems}

For each baseline, relaxed group convolution (\texttt{RGroup}) and relaxed steerable CNN (\texttt{RSteer}), and for each symmetry setting, we conduct a hyperparameter sweep over learning rate, batch size, hidden width, number of layers, and the number of rollout steps used to compute prediction errors during training, using the same search ranges as \citet{Wang2022} (see Table~\ref{tab:hyperparams_smoke}). We also tune the number of filter banks for group-convolution models and the coefficient for the non-equivariance penalty $\lambda_\perp$ for relaxed weight-sharing models. The input sequence length is fixed to $10$. To ensure a fair comparison, we cap the total number of trainable parameters for every model at no more than $10^7$.

\begin{table}[h]
    \centering
    \caption{Hyperparameter tuning range for the asymetric smoke simulation data.}
    \begin{tabular}{|c|c|c|c|c|c|c|}
        \hline
        LR & Batch size & Hid-dim & Num-layers & Num-banks & \#Steps for Backprop & $\lambda_\perp$ \\
        \hline
        $10^{-2} \sim 10^{-5}$ & $8 \sim 64$ & $64 \sim 512$ & $3 \sim 6$ & $1 \sim 4$ & $3 \sim 6$ & $0, 10^{-2}, 10^{-4}, 10^{-6}$ \\
        \hline
    \end{tabular}
    \label{tab:hyperparams_smoke}
\end{table}

\subsection{CT Scan Metal Artifact Reduction}
\label{app:ct_mar}

\begin{figure}[h!]
    \centering
    \includegraphics[width=1.0\linewidth]{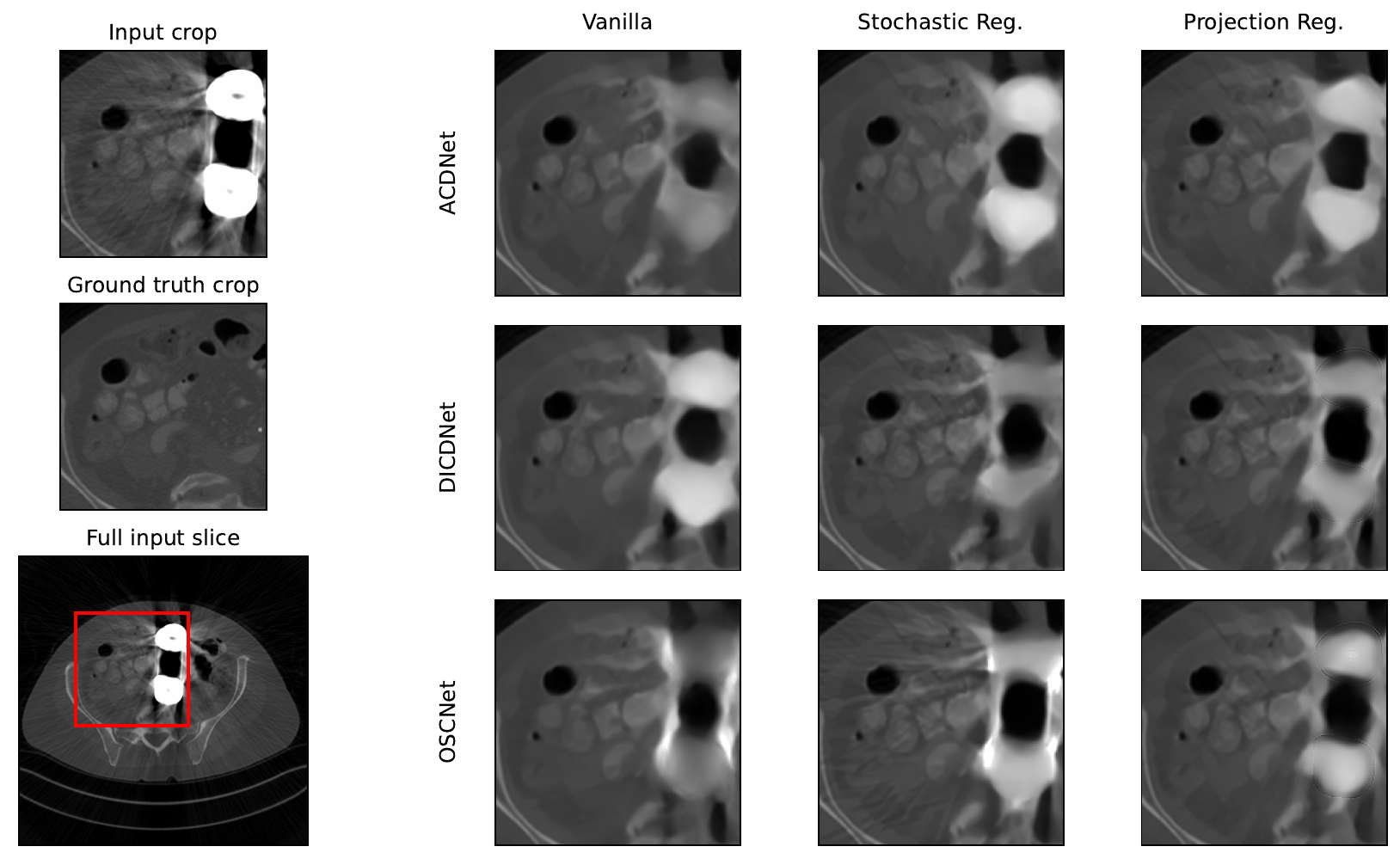}
    \caption{Qualitative comparison of the baseline methods (left column) with each of the sample-based (middle column) and our projection-based regulariser (right column) on the metal artefact reduction task. We show a cropped pelvic slice containing two metallic implants that generate artefacts.}
    \label{fig:ct_mar}
\end{figure}

\subsubsection{Hyperparameters}\label{app:details_ct_mar}

For the most part, we use the same hyperparameters as \cite{Bai_2025_CVPR}. We train for $80$ epochs with a batch size of $12$ for the baselines and our projection-based regulariser, and a batch size of $4$ for the sample-based regulariser. We set the patch size at $256\!\times\!256$. 
Optimization uses Adam \cite{kingma2014adam} ($\beta_1{=}0.5$, $\beta_2{=}0.999$) with initial learning rate $\eta_0{=}2{\times}10^{-4}$ and a MultiStepLR scheduler (milestones at epochs $\{50,100,150,200\}$, decay factor $\gamma{=}0.5$). The model hyperparameters are summarised in Table \ref{tab:hyperparams_ct_mar}.

\begin{table}[ht]
    \begin{center}
        \caption{Hyperparameters for the CT-MAR experiments.}
        \label{tab:hyperparams_ct_mar}
        \begin{tabular}{l@{\hspace{8pt}}c}
        \toprule
        \textbf{Parameter} & \textbf{Value} \\
        \midrule
        $N$ (feature maps) & $8$ \\
        $N_p$ (concat channels) & $35$ \\
        $d$ (dict. filters) & $32$ \\
        Residual blocks / ResNet & $3$ \\
        Stages $T$ & $10$ \\
        \bottomrule
        \end{tabular}
    \end{center}
\end{table}

The scalar weight for sample-based regulariser is set at $10^6$. To set ours, we performed a hyperparameter sweep over the set $\{1.0, 10^{-1}, \dots, 10^{-6}\}$ and chose $\lambda_G=1.0$.

\subsubsection{Projection onto the $C_4$-equivariant kernel subspace}
\label{app:c4_steerable_projection}

We consider steerable CNN layers whose input and output feature spaces are arranged in orientation groups of four (regular-representation channels) for the discrete rotation group
$C_4=\{0,1,2,3\}$ (multiples of $90^\circ$). Let
\[
K \in \mathbb{R}^{C'_{\mathrm{out}}\times C'_{\mathrm{in}}\times 4\times 4 \times s\times s}
\]
denote an $s \times s$ convolution kernel with output block index $p\in\{1,\dots,C'_{\mathrm{out}}\}$,
input block index $q\in\{1,\dots,C'_{\mathrm{in}}\}$, orientation indices
$\alpha,\beta\in\{0,1,2,3\}$, and spatial indices $(i,j)\in\{0,\dots,s-1\}^2$.
Let $S$ be the $4\times 4$ cyclic-shift matrix so that the channel representations of $C_4$ act by
$\rho_{\mathrm{out}}(r)=S^r$ and $\rho_{\mathrm{in}}(r)=S^r$ for $r\in\{0,1,2,3\}$.
Write $\mathrm{rot}_r$ for rotation of the spatial kernel by $90^\circ r$ (with exact
index permutation on the discrete grid).

The natural action of $C_4$ on kernels combines spatial rotation with orientation-channel
permutations:
\begin{equation}
    \label{eq:c4_kernel_action}
    \big(\mathcal{A}(r)\,K\big) \;=\; \rho_{\mathrm{out}}(r)\,\big(\mathrm{rot}_r K\big)\,\rho_{\mathrm{in}}(r)^{-1}
    \;=\; S^r \, \big(\mathrm{rot}_r K\big)\, S^{-r}.
\end{equation}
The orthogonal projector onto this subspace is the (finite) Haar average of the action:
\begin{equation}
    \label{eq:c4_group_average}
    P(K) \;=\; \tfrac14 \sum_{r=0}^{3} S^{r}\, \big(\mathrm{rot}_{r} K\big)\, S^{-r}.
\end{equation}
Index-wise, for any $(p,q,\alpha,\beta,i,j)$, this reads
\begin{equation}
    \label{eq:c4_indexwise}
    \big[P(K)\big]_{p,\alpha;\,q,\beta}[i,j]
    \;=\; \tfrac14 \sum_{r=0}^{3}
    \big[\mathrm{rot}_{r}K\big]_{p,\alpha-r;\,q,\beta-r}[i,j].
\end{equation}
Since~\eqref{eq:c4_group_average} is the average of unitary (permutation + rotation) actions,
$P$ is an orthogonal projector: $P^2=P$ and $P^\top=P$. In practice,
\eqref{eq:c4_group_average} yields an efficient, exact implementation requiring only four
$90^\circ$ rotations and two inexpensive orientation-channel permutations per term.

\subsection{Biomedical Image Classification}

\subsubsection{Hyperparameters}

All baselines (\emph{3D} CNNs, RPPs, and (P)-SCNNs) use the same hyperparameters as in \citet{veefkind24}; see Figure~9 and Tables~12 and~14 in \citet{veefkind24}. In our method, we keep most hyperparameters fixed. We only sweep over the newly introduced parameters \((\lambda_G, \lambda_\perp)\) and tune the learning rate. Table~\ref{tab:hyperparams_medmnist} lists the search ranges; we select the best configuration based on validation metrics. Table~\ref{tab:chosen_hyperparams_medmnist} reports the chosen values.

\begin{table}[h]
    \centering
    \caption{Hyperparameter tuning range for our approach on the MedMNIST datasets.}
    \begin{tabular}{|c|c|c|}
        \hline
        Learning rate & $\lambda_G$ & $\lambda_\perp$ \\
        \hline
        $5 \cdot 10^{-6}, 1 \cdot 10^{-5}, 5 \cdot 10^{-5}, 1 \cdot 10^{-5}, 5 \cdot 10^{-4}$ & $10^{-3}, 10^{-4}$ & $10^{0}, 10^{-1}, 10^{-2}, 10^{-3}$ \\
        \hline
    \end{tabular}
    \label{tab:hyperparams_medmnist}
\end{table}

\begin{table}[t]
    \centering
    \scriptsize
    \setlength{\tabcolsep}{3.2pt}
    \renewcommand{\arraystretch}{1.05}
    \caption{Chosen hyperparameters for each configuration on the MedMNIST datasets.}
    \label{tab:chosen_hyperparams_medmnist}
    \begin{tabular}{@{}l|ccc|ccc@{}}
        \toprule
        \textbf{Parameter} &
        \multicolumn{3}{c|}{$SO(3)$} &
        \multicolumn{3}{c}{$O(3)$} \\
        & \textbf{Nodule} & \textbf{Organ} & \textbf{Synapse}
        & \textbf{Nodule} & \textbf{Organ} & \textbf{Synapse} \\
        \midrule
        Learning rate       & $5 \cdot 10^{-5}$ & $5 \cdot 10^{-5}$ & $10^{-4}$ & $10^{-4}$ & $10^{-4}$ & $10^{-4}$ \\
        $\lambda_{G}$       & $10^{-3}$ & $10^{-3}$ & $10^{-3}$ & $10^{-4}$ & $10^{-3}$ & $10^{-3}$ \\
        $\lambda_{\perp}$   & $1$ & $10^{-2}$ & $1$ & $10^{-1}$ & $10^{-3}$ & $10^{-1}$ \\
        \bottomrule
    \end{tabular}
\end{table}

\subsubsection{Sensitivity with respect to $\lambda_G$ and $\lambda_\perp$}\label{app:sensitivity_lambdas_(s)o3}

We evaluate the sensitivity of our method with respect to the parameters $\lambda_G$ and $\lambda_\perp$. Figure~\ref{fig:sweep_lambdas_medmnist} shows performance on the \texttt{Nodule}, \texttt{Organ}, and \texttt{Synapse} tasks of the MedMNIST dataset \citep{medmnistv2} for $\lambda_G, \lambda_\perp \in \{0.0001, 0.001, 0.01, 0.1, 1.0\}$.

We observe that the optimal parameter choice depends on the degree of equivariance present in the task. \texttt{Nodule} and \texttt{Synapse} benefit from larger values of $\lambda_\perp$, which encourage stronger equivariance, whereas \texttt{Organ} achieves the best performance for smaller $\lambda_\perp$. This is expected because \texttt{Organ} requires distinguishing, for example, left from right kidneys or femurs. This is information that is suppressed under perfect rotational or reflection symmetry.

\begin{figure}
    \centering
    \includegraphics[width=\linewidth]{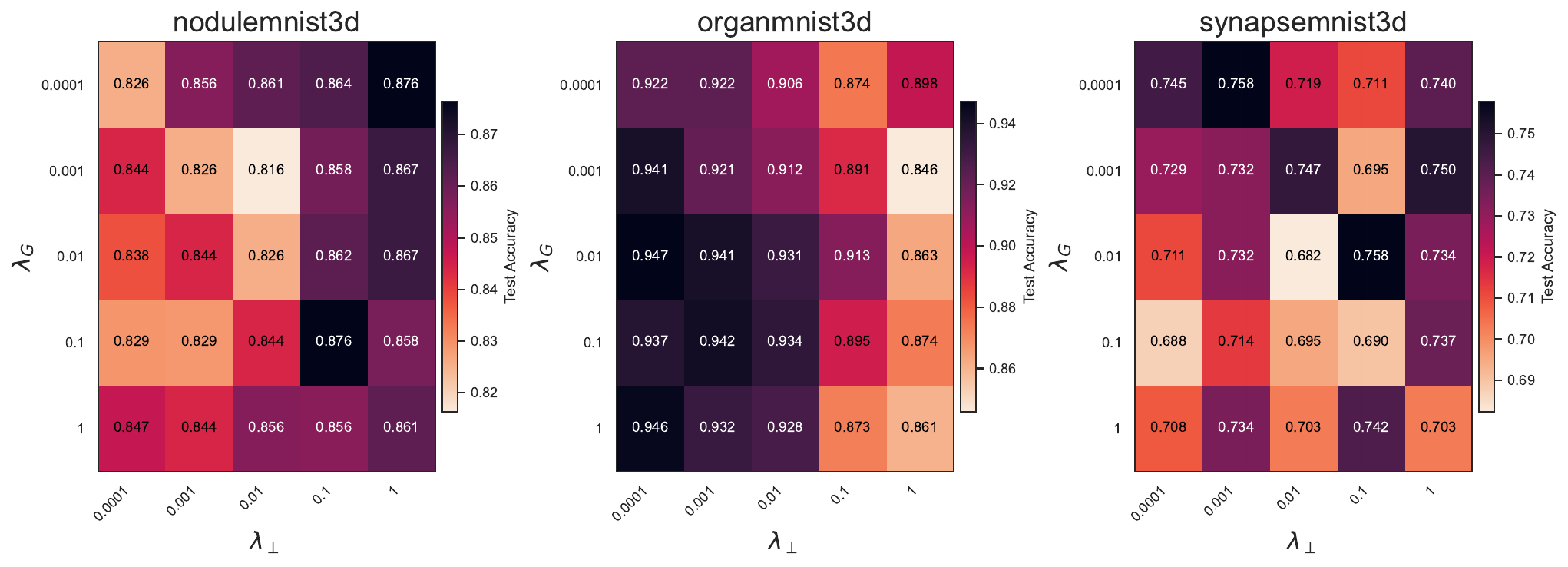}
    \caption{Classification accuracy of our approach on the \texttt{Nodule}, \texttt{Organ}, and \texttt{Synapse} tasks of MedMNIST \citep{medmnistv2} while varying $\lambda_G$ and $\lambda_\perp$.}
    \label{fig:sweep_lambdas_medmnist}
\end{figure}

\subsubsection{Confusion matrices}\label{app:confusion}

\begin{figure}[ht]
  \centering
  \begin{subfigure}[b]{0.32\textwidth}
    \centering
    \includegraphics[width=\linewidth]{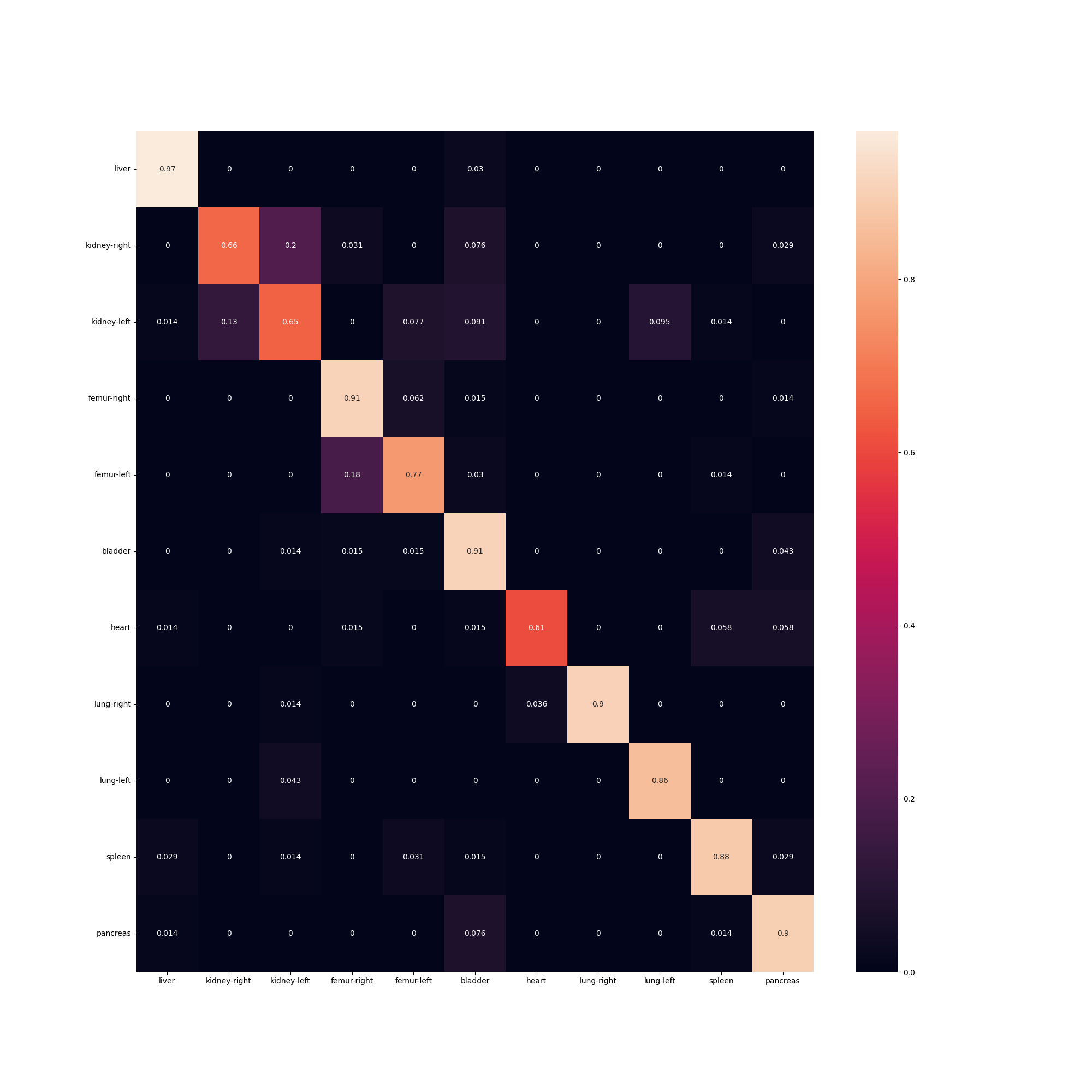}
    \caption{CNN}
    \label{fig:one}
  \end{subfigure}\hfill
  \begin{subfigure}[b]{0.32\textwidth}
    \centering
    \includegraphics[width=\linewidth]{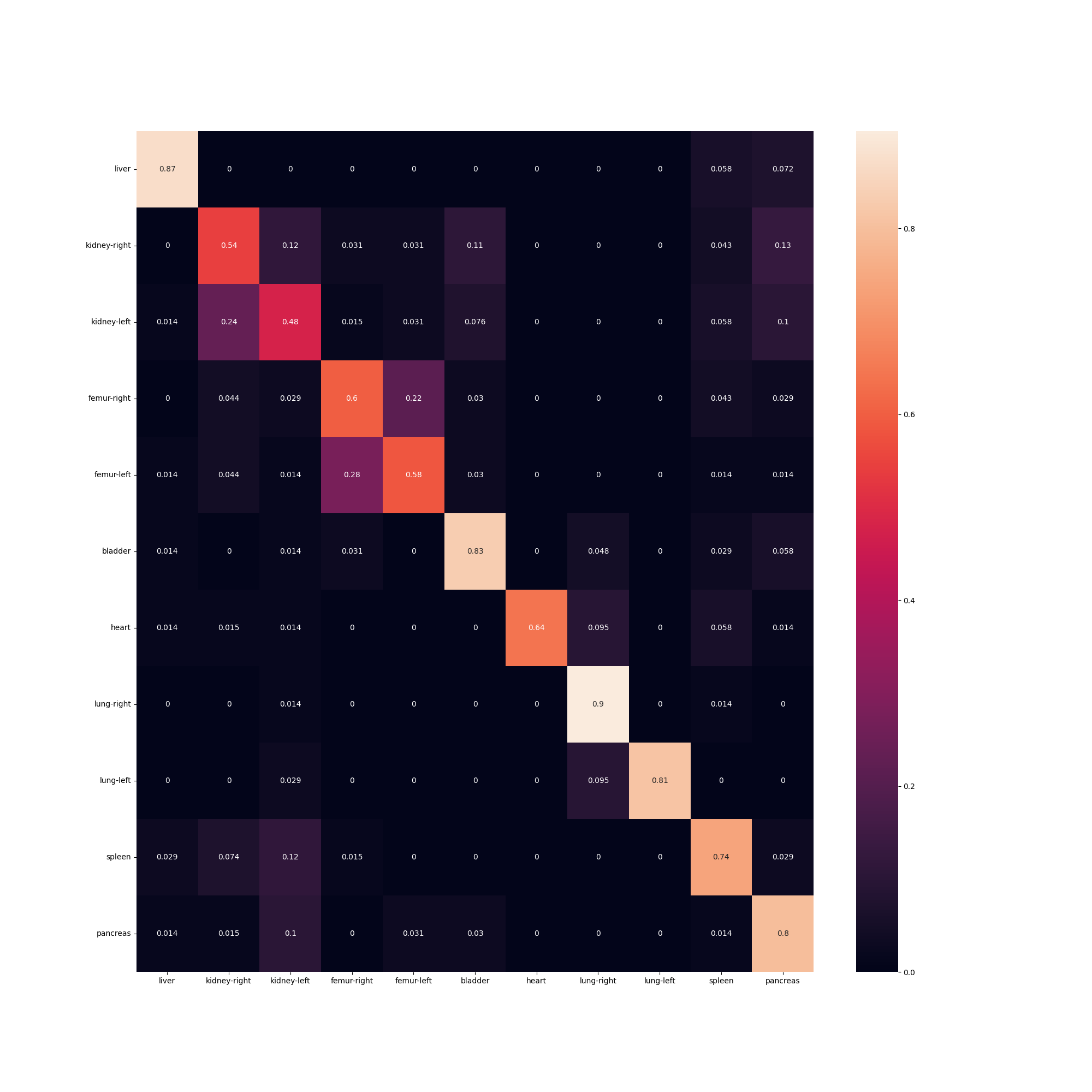}
    \caption{SCNN}
    \label{fig:two}
  \end{subfigure}\hfill
  \begin{subfigure}[b]{0.32\textwidth}
    \centering
    \includegraphics[width=\linewidth]{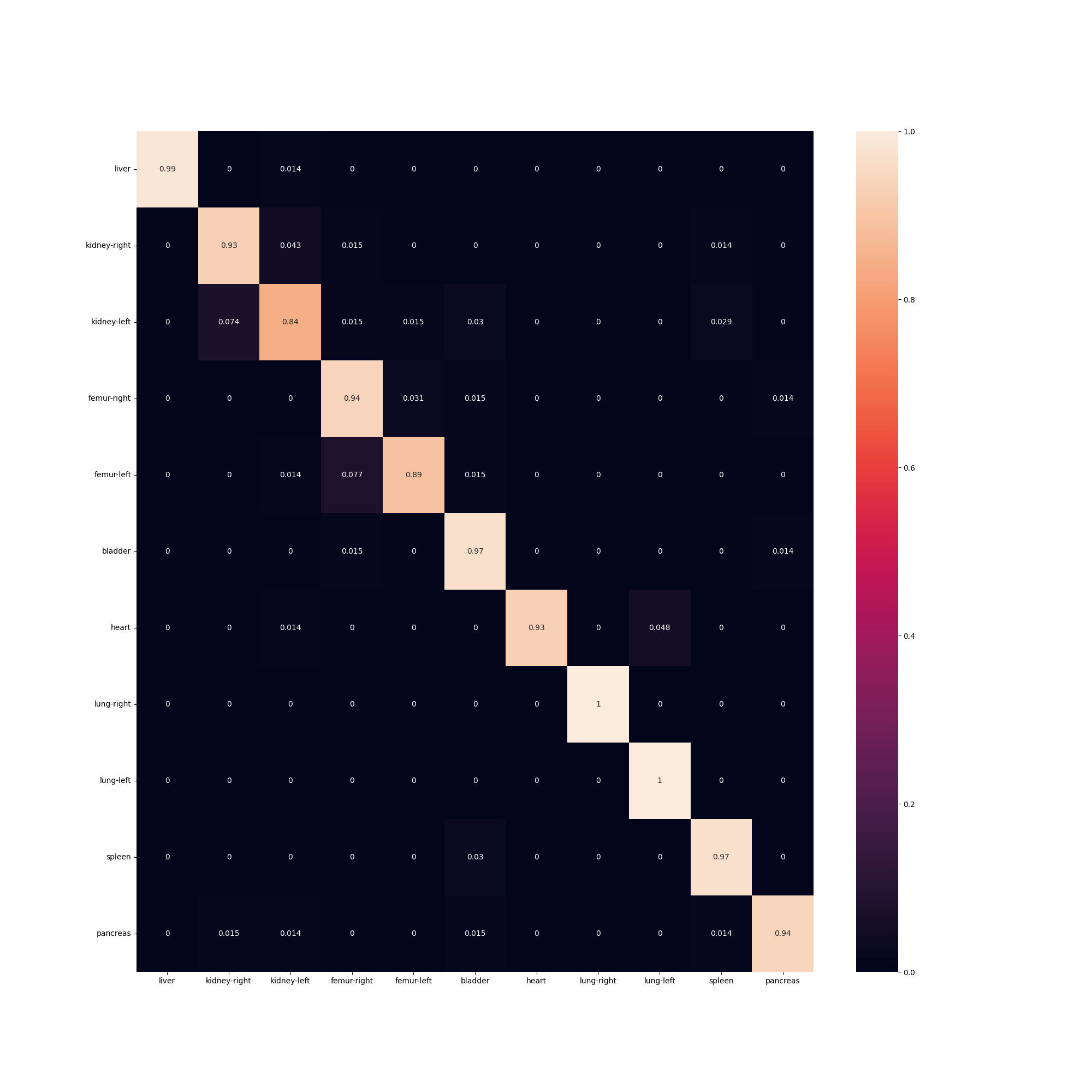}
    \caption{Ours}
    \label{fig:three}
  \end{subfigure}
  \caption{Confusion matrices for a 3D CNN, an SO(3)-equivariant SCNN, and our approach trained on the \texttt{Organ} split of the MedMNIST dataset.}
  \label{fig:confusion_medmnist}
\end{figure}

We examine the confusion matrices (Figure \ref{fig:confusion_medmnist}) for a 3D CNN, an SO(3)-equivariant SCNN, and our approach, all trained on the \texttt{Organ} split of the MedMNIST dataset \citep{medmnistv2}. The SCNN hard-codes rotational equivariance and therefore tends to confuse left and right femurs and kidneys. The 3D CNN exhibits similar confusion for the kidneys, but distinguishes the femurs more reliably. Our method performs best overall, with substantially fewer left–right confusions.





\section{Sensitivity with respect to norm}

\begin{table}[t]
    \centering
    \caption{Results for the CT scan metal artifact reduction task from Section~\ref{sec:experiments_discrete_rotation_equivariance} for different matrix norms. We report throughput during training and inference as well as total epoch time; the performance metrics are PSNR/SSIM. We consider the spectral, Frobenius (which we use by default in Section~\ref{sec:experiments_discrete_rotation_equivariance}) and infinity norms, as well as the $(p,q)$-norms for $p,q\in\{1, 2, 3\}$.}
    \label{tab:norm_comparison}
    \begin{tabular}{l|ccccc}
        \toprule
        Norm &
        \multicolumn{2}{c}{Throughput (no./GPU-s)} &
        Epoch&
        \multicolumn{2}{c}{AAPM} \\
        \cmidrule(lr){2-3}\cmidrule(lr){4-4}\cmidrule(lr){5-6}
        &
        Train $\uparrow$ &
        Inference $\uparrow$ &
        time (s) $\downarrow$ &
        PSNR $\uparrow$ &
        SSIM $\uparrow$ \\
        \midrule
        Spectral     & 6.59 &   10.16     &    877  &     39.25   & 0.9318 \\
        \textbf{Frobenius}& 7.22 & 10.11  &    778  &     38.48   & $\mathbf{0.9457}$ \\  
        Infinity     & 7.73 &   10.12     &    777  &     35.61   & 0.9153 \\
        \midrule
        $(1,1)$      & 7.63 &   10.13     &    785  &     35.57   & 0.8864 \\
        $(1,2)$      & 7.12 &   10.14     &    785  &     38.05   & 0.9365 \\
        $(1,3)$      & 7.12 &   10.13     &    785  &     38.67   & 0.9391 \\
        \midrule
        $(2,1)$      & 7.65 &   10.14     &    783  &     39.33   & 0.9299 \\
        $(2,2)$      & 7.65 &   10.13     &    783  &     38.24   & 0.9430 \\
        $(2,3)$      & 7.61 &   10.14     &    786  &     38.18   & 0.9299 \\
        \midrule
        $(3,1)$      & 7.59 &   10.14     &    787  &     39.10   & 0.9304 \\
        $(3,2)$      & 7.32 &   10.10     &    810  &     39.54   & 0.9322 \\
        $(3,3)$      & 7.18 &   10.16     &    780  &     37.86   & 0.9346 \\
        \bottomrule
    \end{tabular}
\end{table}

In this ablation, we study the impact of the choice of matrix norm in the projection regulariser. We consider the following norms. First, the spectral norm
\begin{equation}
    \|A\|_2 = \max_{\|x\|_2 = 1} \|A x \|_2,
\end{equation}
which is equal to the largest singular value of $A$. Second, the Frobenius norm
\begin{equation}
    \|A\|_F = \sqrt{\sum_{i,j} a_{i,j}^2}.
\end{equation}
Third, the (entrywise) infinity norm
\begin{equation}
    \|A\|_\infty = \max_{i,j} |a_{i,j}|.
\end{equation}
Finally, we consider mixed $(p,q)$-norms, defined row-wise as
\begin{equation}
    \|A\|_{p,q} = \left( \sum_i \left( \sum_j |a_{i,j}|^p \right)^{\frac{q}{p}} \right)^{\frac{1}{q}},
\end{equation}
for $p,q \in \{1,2,3\}$. The corresponding results are reported in Table~\ref{tab:norm_comparison}. We can see that the choice of norm has only a modest effect on both computational cost and reconstruction quality. Training and inference throughput, as well as epoch time, are nearly identical across all norms, except for the spectral norm, which is about $10$--$15\%$ slower per epoch, as expected given the need to estimate the largest singular value. In terms of image quality, several choices yield very similar PSNR/SSIM, with the Frobenius and $(p,q)$-norms for $(p,q)\in{(2,2),(1,3),(3,3)}$ all lying within roughly $1$ dB PSNR and $0.01$ SSIM of each other. Norms that emphasise elementwise extremal behaviour, such as the infinity norm and the $(1,1)$-norm, lead to clear degradation in both PSNR and SSIM, indicating that these penalties are too stiff and effectively underfit the reconstruction task.
Since the spectral norm brings no systematic performance gains while incurring a noticeable runtime overhead, and more aggressive entrywise norms harm reconstruction quality, we adopt the Frobenius norm as our default in Section~\ref{sec:experiments_discrete_rotation_equivariance}.

\section{Declaration of LLM use}

We used LLMs to aid in the writing process for proof-reading, spell checking, and polishing writing.


\end{document}